\documentclass{article}

\usepackage[preprint]{neurips_2026}

\usepackage[utf8]{inputenc} 
\usepackage[T1]{fontenc}    
\usepackage[breaklinks=true,
colorlinks=true,
linkcolor=BrickRed,
urlcolor=Blue,
citecolor=Blue,
anchorcolor=Blue]{hyperref}       

\usepackage{url}            
\usepackage{booktabs}       
\usepackage{amsfonts}       
\usepackage{nicefrac}       
\usepackage{microtype}      
\usepackage{subcaption}

\usepackage{graphicx}
\usepackage{amsmath, amssymb}

\usepackage{multirow}
\usepackage{tablefootnote}

\usepackage[dvipsnames, table]{xcolor}
\usepackage{hhline}
\usepackage{adjustbox} 
\usepackage{setspace}
\usepackage{enumitem}
\usepackage{natbib}
\usepackage{amsthm}
\usepackage{wrapfig} 
\usepackage{xspace}
\usepackage{cleveref}
\usepackage{float}
\usepackage{hyperref}

\newtheorem{proposition}{Proposition}
\newtheorem{corollary}{Corollary}
\newtheorem{definition}{Definition}
\newtheorem{remark}{Remark}
\newtheorem{lemma}{Lemma}

\crefname{proposition}{Proposition}{Propositions}
\Crefname{proposition}{Proposition}{Propositions}
\crefname{corollary}{Corollary}{Corollaries}
\Crefname{corollary}{Corollary}{Corollaries}
\crefname{definition}{Definition}{Definitions}
\Crefname{definition}{Definition}{Definitions}
\crefname{remark}{Remark}{Remarks}
\Crefname{remark}{Remark}{Remarks}
\crefname{lemma}{Lemma}{Lemmas}
\Crefname{lemma}{Lemma}{Lemmas}

\newcommand{\rxrx}{\texttt{RxRx1}\xspace}
\newcommand{\celeba}{\texttt{CelebA}\xspace}

\newcommand{\ie}{\textit{i.e.,}\xspace}
\newcommand{\fext}{feature extractor\xspace}
\newcommand{\fop}{\ensuremath{\Phi}\xspace}
\newcommand{\transl}{translator\xspace}
\newcommand{\Transl}{Translator\xspace}

\newcommand{\trust}{\ensuremath{T}}

\title{Assessing Sample Quality in Conditional Generation under Compositional Shift}

\author{
\textbf{Berker Demirel\textsuperscript{1},
Valentino Maiorca\textsuperscript{1},
Marco Fumero\textsuperscript{1},}
\textbf{Theofanis Karaletsos\textsuperscript{2,3},} \\
\textbf{Francesco Locatello\textsuperscript{1}}
\\[0.5em]
\textsuperscript{1}Institute of Science and Technology Austria (ISTA),\\
\textsuperscript{2}Pyramidal Inc, San Francisco, CA, USA,\\
\textsuperscript{3}Achira Inc, San Francisco, CA, USA
}

\begin{document}

\maketitle

\begin{abstract}

Conditional generators provide a natural tool for controllable generation, including settings where the desired condition is a new composition of observed attributes or experimental factors. In many applications, especially in scientific domains, such models are attractive to explore conditions for which real samples are rare, expensive, or not yet observed. However, this creates a circularity for evaluation: standard conditional quality metrics require a reference target distribution, but in the extrapolative regime that distribution is unavailable by definition.
We address this problem with a post-hoc, per-sample trust score for assessing conditional samples using only the training distribution. The score combines two estimable quantities: global realism, measuring compatibility with the real data manifold, and attribute-wise faithfulness, measuring whether a sample is closer to the requested attributes than to plausible alternatives. 
We show that the score can recover meaningful comparisons across extrapolated generations, under a mild coverage condition on the observed attributes. 
These comparisons enable effective filtering, ranking, and abstention of generations and can be used directly on off-the-shelf pretrained models. 
In biological imaging, selected samples preserve real morphological structure better and improve downstream predictive performance, while similar gains are observed on controlled vision benchmarks. Finally, we show how the score can be applied during generation, enabling abstention before full decoding. 
Code is available at \href{https://github.com/berkerdemirel/faithful-cond-gen}{https://github.com/berkerdemirel/faithful-cond-gen}.
\end{abstract}

\section{Introduction}

\looseness=-1Conditional diffusion models can generate samples for user-specified conditions, including class labels~\citep{dhariwal2021diffusionbeatgans, rombach2022latentdiff}, attribute tuples~\citep{demirel2025morphgen} and biological perturbations~\citep{navidi2025morphodiff}. This is increasingly appealing in scientific settings, where synthetic data are useful not because they reproduce what has already been measured, but because they may help explore conditions that are rare, expensive, or not yet observed~\citep{rxrx1sypetkowski2023,azizi2023synthetic,he2023synthetic}. For example, if we want to achieve certain phenotypic properties in a cell population, we would like to prioritize wet-lab testing of perturbations that are most likely to produce the desired outcome. A generative model can support this process by simulating experiments in silico before laboratory validation. The challenge is determining if these predictions can be trusted, as they must be both visually realistic and faithful to the requested condition, despite the target population being unseen.

%
Evaluating conditional generation faces a circularity issue: standard distributional metrics such as Fréchet Inception Distance (FID)~\citep{heusel2017fid} and Kernel Inception Distance (KID)~\citep{binkowski2018kid}, along with their conditional variants, score generated samples by comparison with real samples from the target distribution. Yet conditional generators are often most useful precisely in settings where such target samples are unavailable. In this extrapolative regime, target-distribution metrics can still serve as controlled validation tools, but they cannot decide whether an individual generated sample should be trusted at deployment time.

\begin{figure}
    \centering
    \includegraphics[
        width=\linewidth,
        trim=15pt 0pt 0pt 10pt,
        clip
    ]{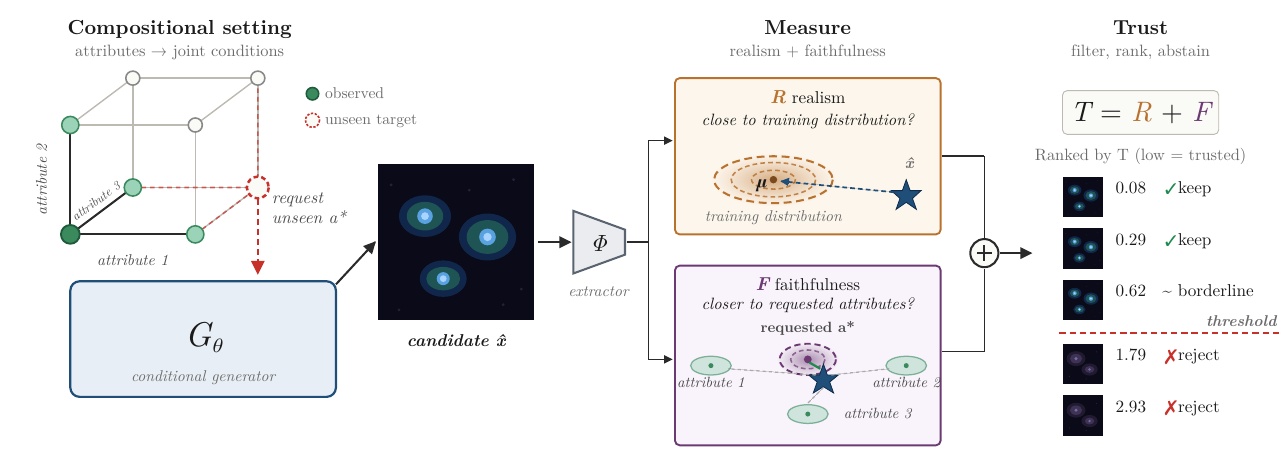}
    \caption{\textbf{Pipeline for trust scoring under compositional shift.} A conditional diffusion model \(G_\theta\) is queried with an unseen joint condition \(a^\star\) and produces candidate samples \(\hat x\). Features from \(\Phi\) are used to compute a realism term \(R\), which measures proximity to the real training distribution, and a faithfulness term \(F\), which measures alignment with the requested attribute values. Their sum \(T=R+F\) provides a calibrated trust score used to rank, filter, or abstain from generated samples.}
    \label{fig:teaser}
\end{figure}

We therefore ask a different question: \emph{what sample quality information is identifiable from the real training distribution alone?} 
We focus on two real-data-calibrated quantities that remain observable. A sample should be \emph{globally realistic}, meaning compatible with the feature geometry of real data, and \emph{attribute-wise faithful}, meaning closer to the requested attribute values than to competing values. These quantities capture separable failure modes: a sample may be realistic but wrong for the requested condition, or condition-consistent but outside the real-data manifold. Existing post-hoc reliability tools typically capture only one side of this split: out-of-distribution and one-class-style scores measure global compatibility with an in-distribution set~\citep{sun2022out, he2023synthetic}, while prompt- or classifier-based alignment scores measure agreement with a requested condition~\citep{muller2025mahalanobis,liu2020energy,ramos2025beyond}. 

Our method combines both views in a post-hoc trust score computed in the latent space of an auxiliary \fext, using only the generated sample and real training-set statistics.
The key idea is attribute-level decomposition: instead of matching a sample to an unseen condition as a whole, we evaluate each constituent attribute independently, checking whether its requested value is favored over plausible alternatives in observed real-data contexts.
Under this coverage condition,  we show that attribute-wise decisions remain identifiable even when the full joint target is not. This formalizes both what the score can recover and what remains fundamentally unidentifiable without target samples.

Most notably, we demonstrate that the same real-data geometry allows our score to be evaluated before generation is complete. By aligning the diffusion model to the \fext space, we are able to evaluate the same trust score directly in representation space, during generation. This avoids the computational cost of repeated forward passes through the \fext and enables low-quality samples to be discarded before full generation.


Our contributions can be summarized as follows:
\begin{itemize}[leftmargin=*]
    \item \textbf{Reference-free trust score.} We recast per-sample quality assessment for conditional generators as a problem solvable without observing the joint target conditioning. Decomposing it into two geometric quantities, global realism and attribute-wise faithfulness, yields our post-hoc trust score computable from a pretrained generator, an encoder, and training-set statistics.
    \item \textbf{Attribute-level identifiability.} We show that meaningful attribute-level comparison scores can be identified even when the full joint target is not.
    \item \textbf{Empirical validation.} On \rxrx and \celeba our score reliably filters, ranks, and curates samples under support shift under both KID-based and downstream task-based validations. 
    \item \textbf{Early sample rejection}. We show how the score can be lifted from a post-hoc method to the latent representations inside the generator itself via a simple geometric alignment objective, enabling quality assessment and abstention before decoding.
\end{itemize}

\section{Related Work}

\paragraph{Image quality metrics.}
\looseness=-1
There is active literature evaluating generative models by comparing real and generated sample sets in a feature space, including FID~\citep{heusel2017fid}, KID~\citep{binkowski2018kid}, density/coverage~\citep{naeem2020fidelity}, probabilistic precision/recall~\citep{park2023probabilistic,sajjadi2018assessing, kynkaanniemi2019improved}, as well as sample-level metrics closer to our setting~\citep{alaa2022faithful}. However, these metrics require a collection of real reference samples from the target distribution being evaluated. In contrast, our method combines capabilities that no other method offers as a set: it is per-sample, post-hoc (\ie after generation, and therefore, model agnostic), and most importantly, it is designed for compositional generalization, where no real samples from the requested target condition are available.

\paragraph{Compositional generalization in conditional diffusion.}

A growing line of work studies whether conditional diffusion models can generate unseen compositions of attributes or concepts. The closest connection to our setting is \cite{okawa2023compositional}, which analyzes when compositional abilities emerge in a synthetic discrete conditional-generation task. Related work studies the emergence of hidden compositional capabilities in concept space~\citep{park2024emergence}, and analyzes hierarchical compositional generalization in diffusion models through context-free grammars~\citep{favero2025compositional}. Other work modifies training objectives to encourage logical or compositional generation under partial support~\citep{gaudicoind}. Our focus is complementary. We do not try to improve the generator or characterize when compositional generalization emerges. Instead, we ask how to evaluate and rank individual generated samples after training, especially when the requested joint condition has no real target reference set. This shifts the question from whether a model can sometimes extrapolate to whether a particular extrapolated sample should be trusted.

\paragraph{OOD detection and support scoring.}
\looseness=-1
Out-of-distribution detection and feature-space support estimation provide another line of related work, but they are primarily developed for discriminative prediction rather than as trust mechanisms for synthetic generations. Mahalanobis-based detectors measure compatibility with an in-distribution feature geometry~\citep{lee2018simple,muller2025mahalanobis}, distance-based methods provide nonparametric alternatives~\citep{sun2022out, rmahren2021simple, nnguidepark2023, FDBDliu2024}, and energy-based methods score whether an input lies in a high-confidence region of a learned representation space~\citep{liu2020energy}. We repurpose ~\cite{lee2018simple} to build a trust mechanism for conditional diffusion models. For realism, we ask whether a sample looks like real data from the domain; for faithfulness, we ask whether, within a real-data-calibrated geometry, it is closer to the requested attribute value than to its competitors. This complements unconditional support scoring by incorporating the requested conditioning, even when the requested joint target is unavailable.

\paragraph{Faithfulness and alignment scores for generated data.}
Conditional faithfulness is often evaluated with discriminative surrogates. In compositional-generation studies, probes can test whether generated samples contain requested attributes~\citep{okawa2023compositional}. In text-to-image settings, CLIP-style alignment is widely used as a proxy for prompt adherence, including when assessing whether synthetic images are useful for recognition tasks~\citep{he2023synthetic}. Recent work has also moved alignment scoring into the denoising process, evaluating prompt--latent agreement before a final image is decoded~\citep{ramos2025beyond}. These approaches are useful discriminative tests of whether a generated sample expresses a requested label, prompt, or attribute, but they are not directly image-quality metrics: an image can be classified as condition-consistent while remaining unrealistic or far from the real data geometry. Our goal is to identify what conditional quality remains estimable from the training distribution alone when the requested joint target is unavailable. This yields a decomposition into global realism and reference-anchored per-attribute faithfulness, where faithfulness is not a standalone classifier score but a conditional component of a real-data-calibrated quality score.


\section{Method}
\label{sec:method}

We propose a post-hoc per-sample \emph{trust score} for conditional generation when the requested joint condition may be absent from the real training support. The score is computed w.r.t. any individually generated sample, relying only on the training data and a pretrained \fext \fop. It combines two measurable quantities: \emph{global realism} (\ie compatibility with the real-data feature geometry), and \emph{attribute-wise faithfulness},  (\ie agreement with each requested attribute value relative to competing ones). While the proposed score is applicable to any pretrained conditional generative model, in the following we assume a latent conditional diffusion model~\citep{rombach2022latentdiff}. \Cref{fig:teaser} illustrates our pipeline.

\paragraph{Generative setting.}
Let \(\mathcal A_1,\dots,\mathcal A_K\) be finite attribute-value spaces, and let \(a^\star=(a_1^\star,\dots,a_K^\star)\in
\mathcal A_1\times\cdots\times\mathcal A_K\) denote the requested joint condition.
We consider a latent diffusion model $G_{\theta}$ with latent space \(\mathcal Z\), input space \(\mathcal X\), decoder \(D:\mathcal Z\to \mathcal X\), and denoising network \(\epsilon_\theta\). Starting from Gaussian noise \(z_T\in \mathcal Z\), the reverse process conditioned on \(a^\star\) produces a final latent \(\hat z_0(a^\star)\in \mathcal Z\), which is decoded into the generated sample \(\hat x(a^\star) = D(\hat z_0(a^\star)) \in \mathcal X\).

The joint condition \(a^\star\) may appear in training or may lie under compositional support shift, where the specific tuple is unseen but its individual attributes appear in other real training samples.
%
Let \(\fop:\mathcal X\to\mathbb R^d\) denote the \fext.
We score samples using normalized features \(y(x)=\fop(x)/\|\fop(x)\|_2\). All means, covariance estimates, precision matrices, and calibration constants below are fit only on normalized real training features. 

\paragraph{Global realism.}
Let \(\mu_{\mathrm{real}}\) and \(\Sigma_{\mathrm{real}}\) be the empirical mean and regularized covariance of real normalized features. We define the global Mahalanobis energy
\[
E_{\mathrm{real}}(y)
=
(y-\mu_{\mathrm{real}})^\top
\Sigma_{\mathrm{real}}^{-1}
(y-\mu_{\mathrm{real}}),
\]
and standardize it as \(R(y)=(E_{\mathrm{real}}(y)-m_R)/s_R\), where \(m_R\) and \(s_R\) are the mean and standard deviation of \(E_{\mathrm{real}}\) on real calibration features. Therefore, larger \(R\) indicates lower realism.

\paragraph{Attribute-wise faithfulness.}
For each attribute \(k\) and value \(v\in\mathcal A_k\), define the real-data prototype \(\eta_{k,v}=\mathbb E[y\mid a_k=v]\), estimated by averaging the normalized features \(y(x)\) of all real training samples whose \(k\)-th attribute value is \(v\).
For each attribute \(k\), we estimate a shared precision matrix \(P_k\succ0\) from the pooled within-value covariance of real samples and define
\[
d_k(y;v)
=
(y-\eta_{k,v})^\top P_k (y-\eta_{k,v}).
\]
For requested value \(t\in\mathcal A_k\), the target-vs-competitor margin is
\[
M_k(y;t)=d_k(y;t)-\min_{v\neq t}d_k(y;v).
\]
Thus \(M_k(y;t)<0\) means that \(y\) is closer to the requested value than to every competing value of attribute \(k\).
We standardize each margin on real samples with \(a_k=t\): \(F_k(y;t)=(M_k(y;t)-m_{k,t})/s_{k,t}\), where \(m_{k,t}\) and \(s_{k,t}\) are the corresponding real calibration mean and standard deviation. The total faithfulness score is \(F(y;a^\star)=\frac{1}{K}\sum_{k=1}^K F_k(y;a_k^\star)\). Larger \(F\) indicates weaker agreement with the requested attributes.

The final trust score is $\trust(y;a^\star)=R(y)+F(y;a^\star).$ 
The decomposition is interpretable: \(R\) detects samples that are globally far from the real feature distribution, while \(F\) detects samples that are realistic but mismatched to the requested attributes. Estimation details are given in \Cref{app:estimation-details}.

\subsection{When can we assess sample quality in practice?}
\label{sec:feature-acquisition}

We remark that the same score \trust{} can be used both after generation and during generation; only the way we obtain the \fop-compatible features \(y\) changes. \Cref{fig:mod_pipeline} in Appendix \ref{app:model-pipeline} illustrates the difference. 

\paragraph{Post-generation scoring.}
After generation, we decode the image and apply the \fext:
\[
y^{\mathrm{post}}
=
\frac{\fop(D(\hat z_0(a^\star)))}{\|\fop(D(\hat z_0(a^\star)))\|_2}.
\]
We then evaluate \(\trust(y^{\mathrm{post}};a^\star)\). This directly scores the generated image in the feature-space geometry, but requires the VAE decoder and a \fop pass.

\paragraph{During-generation scoring.}
\label{par:during-generation}
While generating, we skip decoding and \fext encoding by mapping internal denoising representation into a \fop-compatible space. Let \(r_{\ell,\tau}(a^\star)\in\mathbb R^{d_{\ell}}\)
denote the representation extracted at layer \(\ell\) and denoising step \(\tau\) along the inference trajectory conditioned on \(a^\star\). We learn a \emph{\transl}, parametrized by a shallow MLP, \(g_{\phi,\ell}:\mathbb R^{d_\ell}\to\mathbb R^d\)
to match the pairs \(\bigl(r_{\ell,\tau}(x_i,a_i),\, y(x_i)\bigr)\) obtained from real training samples \(x_i\) with observed attribute condition \(a_i\).
\[
y_{\ell,\tau}^{\mathrm{map}}
=
\frac{g_{\phi,\ell}(r_{\ell,\tau}(a^\star))}
{\|g_{\phi,\ell}(r_{\ell,\tau}(a^\star))\|_2}.
\]
At inference time, we evaluate \(\trust(y_{\ell,\tau}^{\mathrm{map}};a^\star)\). Thus post-generation and during-generation scoring use the same real-data-calibrated metric; they differ only in whether the feature is obtained from the decoded image or from the denoising trajectory.

\paragraph{Online abstention.}
\label{sec:abstention-method}
The mapped score enables early rejection during denoising. Given a layer-step-specific threshold \(\kappa_{\ell,\tau}\), we abstain when
\[
\trust(y_{\ell,\tau}^{\mathrm{map}};a^\star)>\kappa_{\ell,\tau}.
\]
This rejects samples predicted to have low trust before completing the full denoising trajectory, decoding, and post-generation encoder pass. Details of the \transl training and compute accounting are given in \Cref{app:compute-details}.

\section{Theory: Identifiability under missing target distributions}
\label{sec:theory}
If the requested condition \(a^\star\) is missing from the real data, then the behavior of \(y\) under that condition cannot be estimated directly from the data.

\begin{proposition}[Informal: non-identifiability of missing targets]
\label{prop:missing-target-informal}
Let \(S\subseteq\mathcal A_1\times\cdots\times\mathcal A_K\) be the observed support of
the conditioning variable \(a\), and let \(a^\star\notin S\). From the observed joint
distribution of \((y,a)\) restricted to \(S\), the conditional distribution of \(y\) at \(a^\star\) is not identifiable without assumptions relating observed and unobserved conditions. Therefore, no post-hoc score using only the observed real distribution can certify full conditional fidelity to the missing target distribution.
\end{proposition}
This motivates a weaker target: attribute-level comparisons. Instead of asking whether a sample matches the unobserved joint condition \(a^\star\), we ask whether it is closer to each requested attribute value than to competing values. 

\begin{definition}[Reference coverage]
\label{def:reference-coverage}
Let \(a_{-k}\) denote all attributes except \(a_k\). The observed support has
\emph{reference coverage} if there exists a reference condition \(\bar a\) such that,
for every attribute \(k\) and value \(v\in\mathcal A_k\),
\[
\mathbb P(a_k=v,\;a_{-k}=\bar a_{-k})>0.
\]
\end{definition}

Reference coverage provides just enough conditioning variation to isolate attribute effects without relying on structural assumptions. When attributes only appear in fixed combinations, their effects are support-confounded: the data cannot tell which attribute accounts for an observed change. Reference coverage avoids this ambiguity while using the smallest possible set of observed joint conditions. We formalize this minimality result and give a support-confounding example in \Cref{app:reference-coverage-minimal-proof,app:support-confounding}.



Moreover, under reference coverage, we can identify a fixed-context attribute comparator. Define the reference prototype
\(\mu^{\mathrm{ref}}_{k,v}:=\mathbb E[y\mid a_k=v,a_{-k}=\bar a_{-k}]\). For a fixed \(P_k\succ0\), define
\(d_k^{\mathrm{ref}}(y;v):=(y-\mu^{\mathrm{ref}}_{k,v})^\top P_k(y-\mu^{\mathrm{ref}}_{k,v})\), and for target value \(t\),
\[
M_k^{\mathrm{ref}}(y;t)
:=
d_k^{\mathrm{ref}}(y;t)-\min_{v\neq t}d_k^{\mathrm{ref}}(y;v).
\]
A negative margin means that \(y\) is closer to the requested value \(t\) than to every competitor for attribute \(k\), in the reference context.

\begin{proposition}[Point identification of the reference-anchored comparator]
\label{prop:anchored-identification}
Assume reference coverage, let \(y\) be integrable, and fix \(P_k\succ0\). Then \(\mu^{\mathrm{ref}}_{k,v}\), \(d_k^{\mathrm{ref}}(y;v)\), and \(M_k^{\mathrm{ref}}(y;t)\) are point-identified from the observed joint distribution of \((y,a)\).
\end{proposition}

See \Cref{app:well-definedness} for the proof. Thus, missing target samples prevent full conditional evaluation, but not all attribute-level evaluation. Reference coverage identifies an anchored comparator.

\emph{\textbf{Remark}}:
The implemented score uses pooled prototypes
\(\eta_{k,v}:=\mathbb E[y\mid a_k=v]\) rather than
\(\mu^{\mathrm{ref}}_{k,v}\), because pooling uses all real samples with
\(a_k=v\) and is more sample-efficient. \Cref{app:pooled-reference-prototypes}
gives sufficient conditions under which the pooled comparator makes the same
attribute-wise decisions as the reference-anchored comparator, and
\Cref{app:perturbation-bound-empirical} reports the corresponding empirical
agreement.

\looseness=-1
Taken together, these results clarify what incomplete conditional support can and cannot identify. Missing joint targets rule out certification of full conditional fidelity without additional assumptions. Reference coverage provides a cardinality-minimal support condition for unconfounded attribute-level comparisons, under which the reference-anchored comparator is point-identified. The implemented pooled score is a sample-efficient surrogate: the appendix gives sufficient conditions for agreement with the identified comparator and empirically verifies this agreement. Additionally, \Cref{app:full-feature-discriminants} shows that the comparator can still be evaluated in the full feature space without requiring a disentangled subspace, as the decisions depend only on prototype-difference discriminant directions.

\section{Experiments}
\label{sec:experiments}
We evaluate the trust score in the regime it is designed for: the scorer is calibrated on real training-support features and never uses real samples from the requested target condition at deployment time. See~\Cref{app:impl-details} for the implementation details.

\paragraph{Datasets and shift regimes.}
We choose the fluorescent microscopy images dataset \rxrx~\citep{rxrx1sypetkowski2023} as the main scientific benchmark. It contains 4 cell types and 1138 siRNA perturbations, yielding 4552 cell-type$\times$perturbation attribute combinations (\ie, conditions). 
After filtering by support to obtain enough real samples for reliable evaluation, we obtain a test set of 25 conditions unseen at training time and 25 seen ones; full construction details are in \Cref{app:rxrx1-subset}.
%
\looseness=-1
For a more controlled setting we used a subset of \celeba~\citep{liu2015celeba} which is a compositional testbed with 4 binary attributes (Male, Smiling, Blond\_Hair, Eyeglasses), yielding $16$ conditions. In this specific setting, training sees only the all-zero reference and the four single-attribute conditions, while all $16$ conditions are queried at generation time. Results for models trained with the full conditional coverage are reported in \Cref{app:scorer-ablation}.

\paragraph{Models and scoring spaces.}
For each dataset and support regime, we train a vanilla conditional diffusion model (SiT-B/2 \cite{ma2024sit}) and two REPA~\citep{yu2025representation} variants aligned to DINOv3~\citep{simeoni2025dinov3} or SigLIP~\citep{zhai2023siglip}. The vanilla model is the basic deployment target of the trust score, while the REPA variants test transfer across diffusion models and provide a natural baseline for model-internal scoring due to their alignment training objective (\Cref{sec:during-gen}). Post-generation trust scores are computed in the \fext representation space: DINOv3 for vanilla models and the corresponding teacher representation for REPA models. During generation, scores are computed using a \transl that aligns the diffusion model features at a given layer, to SigLIP features. Additionally, we show a negative result using OpenPhenom~\citep{ophenomkraus2024} features, due to representation collapse in \Cref{app:openphenom-diagnostic}

\paragraph{Validation metrics.}
Generation quality is measured via KID, which assumes access to the real samples from the unseen target conditioning as ground truth, as opposed to our trust score.
\[
\Delta\mathrm{KID}
=
\mathrm{KID}(\mathrm{gen},\mathrm{real}_A)-\mathrm{KID}(\mathrm{real}_B,\mathrm{real}_A),
\]
always evaluated in a fixed DINOv3 mean-patch feature space. On \rxrx, we additionally validate trust selection in CellProfiler (CP)~\citep{cellprofilercarpenter2006} morphology space, which provides a biologically interpretable per-cell measurements independent of DINOv3.


\subsection{Post-generation trust as data curation}
\label{sec:trust-ordering}

We first evaluate the score in the post-generation setting. The central question is whether a score calibrated only on real training-support data can curate generated samples for requested conditions whose target distribution is unavailable at scoring time.

\subsubsection{Real-data-calibrated filtering and condition ranking}
\label{sec:fpr95}

Following the convention commonly used in the out-of-distribution detection literature with the FPR95 metric~\citep{sun2022out, demirel2026outofdistribution, muller2025mahalanobis, liu2020energy}, we instantiate the abstention rule from \Cref{sec:abstention-method} by setting the acceptance threshold at the 95th percentile of real-sample trust scores and comparing the accepted subset against a condition-matched random subset of the same size. \Cref{tab:fpr95-selection} compares the $\Delta$KID of selected generated samples based on trust score (denoted as $\Delta$KID$_{trust}$) against a condition-matched random subset of generated samples of the same size (denoted as $\Delta$KID$_{baseline}$). It reports the percentage of generated samples that pass the FP95-real trust threshold, and the per-condition Spearman correlation between mean trust score and \(\Delta\)KID.


\begin{table}[t]
\centering
\caption{Support-shift post-generation results. P95-real-threshold sample selection uses a threshold set at the 95th percentile of real-sample trust scores; $\Delta$KID is measured in DINOv3 space and compared to a condition-matched random subset. $\Delta$\% is the relative $\Delta$KID improvement of trust-selected over random baseline (positive is better). $\rho(\trust)$ is the condition-level Spearman correlation between mean trust and $\Delta$KID. Component correlations and full-support rows are in \Cref{app:detailed_results}.}
\label{tab:fpr95-selection}
\small
\begin{tabular}{c l cccc c}
\rowcolor{gray!20}\toprule
Dataset & Model & Accept\% & $\Delta$KID$_{\text{trust}}$$\downarrow$ & $\Delta$KID$_{\text{baseline}}$$\downarrow$ & $\Delta$\%$\uparrow$ & $\rho(\trust)$$\uparrow$ \\
\midrule
\multirow{3}{*}{\celeba}
 & Vanilla          & 58.7 & 0.224{\tiny$\pm$.009} & 0.368{\tiny$\pm$.020} & $+$39.1 & 0.96 \\
 & REPA (DINOv3)    & 55.4 & 0.228{\tiny$\pm$.010} & 0.401{\tiny$\pm$.031} & $+$43.1 & 0.96 \\
 & REPA (SigLIP)    & 56.1 & 0.239{\tiny$\pm$.018} & 0.423{\tiny$\pm$.020} & $+$43.6 & 0.96 \\
\midrule
\multirow{3}{*}{\rxrx}
 & Vanilla          &  4.4 & 0.196{\tiny$\pm$.006} & 0.323{\tiny$\pm$.013} & $+$39.4 & 0.90 \\
 & REPA (DINOv3)    &  5.5 & 0.197{\tiny$\pm$.006} & 0.333{\tiny$\pm$.019} & $+$40.8 & 0.90 \\
 & REPA (SigLIP)    &  6.1 & 0.181{\tiny$\pm$.003} & 0.312{\tiny$\pm$.019} & $+$42.0 & 0.87 \\
\bottomrule
\end{tabular}
\end{table}

\looseness=-1
The trust score consistently selects better samples under support shift. On \celeba, the selected subsets improve $\Delta$KID by 39--44\% while retaining roughly 55--59\% of generated samples. On \rxrx, selection is much more stringent, accepting only 4--6\% of samples, but the $\Delta$KID improvement remains strong at 39--42\%. Beyond sample selection, the same table shows that per-condition mean trust provides a strong ranking of condition difficulty in both domains. Accordingly, we report the full realism/faithfulness decomposition in \Cref{app:detailed_results} and defer the head-to-head scorer ablation to \Cref{app:scorer-ablation}.

\noindent\textbf{Takeaway.} Our score provides a practical support-shift filter and a useful condition-level ordering without using target real samples at scoring time.

\subsubsection{Trust rankings track sample quality and downstream utility}
\label{sec:downstream-ordering}

We next show how our score can be employed as a full ranker, rather than a filterer. We partition generated samples into class-balanced trust deciles, train a condition classifier on each decile separately, and evaluate on held-out real data. \Cref{fig:celeba-decile} shows the controlled CelebA held-out worse trust deciles have worse $\Delta$KID and yield less useful training data for downstream condition classification.

\begin{figure}[t]
\centering
\begin{subfigure}[t]{0.48\textwidth}
    \centering
    \includegraphics[width=\textwidth]{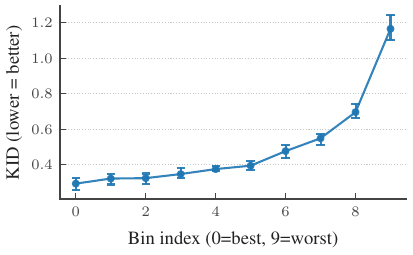}
    \caption{CelebA $\Delta$KID by trust decile.}
    \label{fig:celeba-decile-kid}
\end{subfigure}
\hfill
\begin{subfigure}[t]{0.48\textwidth}
    \centering
    \includegraphics[width=\textwidth]{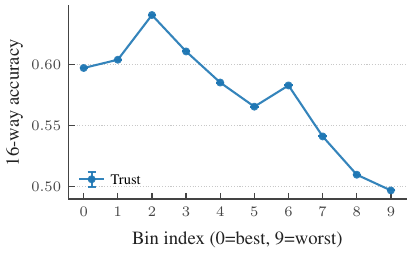}
    \caption{CelebA downstream 16-way accuracy.}
    \label{fig:celeba-decile-downstream}
\end{subfigure}
\caption{\textbf{CelebA decile binning (REPA-DINOv3 held-out, DINOv3 scoring).} $\Delta$KID increases monotonically from bin~0 (best trust) to bin~9 (worst) (\emph{left}), and correlates with downstream classification accuracy drops (\emph{right}). Binning results with faithfulness and realism components, together with RxRx1 DINOv3 decile curves are reported in \Cref{app:detailed_results}}
\label{fig:celeba-decile}
\end{figure}

On CelebA, $\Delta$KID degrades by roughly $4\times$ from the best to worst decile, while downstream accuracy decreases by about 10\% when training on low-quality samples according to trust. This confirms that the score is ordinally meaningful in a setting where the support shift is controlled by design. In the more realistic scientific setting, the analogous \rxrx DINOv3-space decile curves show the same ordering for KID and downstream classification; we report them in \Cref{fig:rxrx1-decile} of \Cref{app:detailed_results}. 

\noindent\textbf{Takeaway.} Across all tested settings, trust score induces a sample-level ordering that tracks both image quality and downstream utility, rather than simply serving as a thresholding rule.

\subsubsection{Biological validation with CellProfiler morphology}
\label{sec:cp-validation}
To test whether sample selection based on trust score captures biologically meaningful morphology, we evaluate \rxrx generations in CellProfiler \cite{cellprofilercarpenter2006} (CP) feature space: interpretable per-cell measurements computed directly from segmented images. We use 621 CP features after variance thresholding, and a real-data $|z|\!\leq\!5$ outlier filter, then standardize on real data. The full CP pipeline is in \Cref{app:cp-validation}.



\paragraph{CP-space decile utility.}
We repeat the decile downstream experiment of Figure \ref{fig:celeba-aligned-kid}, on \rxrx, using CP feature space as opposed to DINOv3. \Cref{fig:cp-decile-repa-siglip-main} shows results with the REPA-SigLIP diffusion model: better trust deciles yield better classifiers for both cell type classification (4 classes) and cell-type$\times$perturbation targets (50 classes in total) show a similar trend. The same pattern holds for all three marginal generators and both targets in \Cref{app:cp-validation}.
\begin{figure}[t]
\centering
\includegraphics[width=\textwidth]{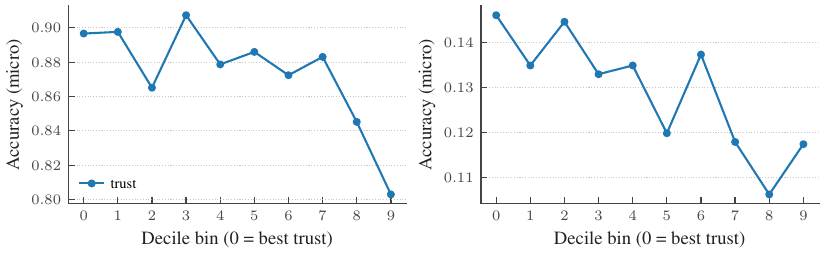}
\caption{\textbf{Main RxRx1 CellProfiler validation (REPA-SigLIP marginal, SigLIP trust scoring).} CP-space downstream classification by trust decile. \emph{Left}: 4-way cell-type accuracy. \emph{Right}: 50-way condition accuracy. Trust-ranked deciles show a clear correlation with the classification performance, showing that trust ordering improves utility in an interpretable morphology space independent of the DINOv3 validation encoder.}
\label{fig:cp-decile-repa-siglip-main}
\end{figure}
%
%
%
%
%
%
%
%
\paragraph{Real-spread-normalized centroid distance.}
For each generated sample at condition $(c,s)$, we compute its CellProfiler distance to the real-condition centroid and standardize it by real-real variation:
\[
 z(x)=
 \frac{
 \lVert \mathrm{CP}(x)-\mu_{\mathrm{real},c,s}\rVert_2
 -
 \mu_{\mathrm{RR}}(c,s)
 }{
 \sigma_{\mathrm{RR}}(c,s)
 }.
\]
Here $\mu_{\mathrm{RR}}(c,s)$ and $\sigma_{\mathrm{RR}}(c,s)$ are estimated from random half-splits of real samples within the same condition. Thus $z=0$ corresponds to typical real-real morphology spread, while $z>0$ indicates excess deviation from the real morphology centroid.


\Cref{tab:cp-centroid} shows that trust selection improves CellProfiler morphology alignment in all six model/split settings, with bootstrap confidence intervals strictly below zero (\Cref{app:cp-validation}, \Cref{tab:cp-centroid-full}). The seen split is nearly real-like after selection: REPA-DINOv3 trust-selected samples are only $0.024\sigma$ beyond typical real morphology. On unseen conditions, generated samples drift farther from the real distribution, but trust selection still removes 18--29\% of the standardized excess distance.

\begin{wraptable}[13]{r}{0.64\textwidth}
\vspace{-1.0em}
\centering
\caption{Real-spread-normalized CP morphology distance on RxRx1 (kept-621 CP features). $z(x)$ is the distance to the matched real-condition centroid, standardized by its real-real morphology variation. Accept\% is the share of generations passing the P95-real threshold. Bootstrap 95\% CIs are in \Cref{tab:cp-centroid-full}.}
\label{tab:cp-centroid}
\small
\begin{adjustbox}{width=\linewidth}
\begin{tabular}{ll c ccc}
\rowcolor{gray!20}\toprule
Model & Split & Accept\% & $\bar z_{\mathrm{trust}}$$\downarrow$ & $\bar z_{\mathrm{baseline}}$$\downarrow$ & $\Delta \%$$\uparrow$ \\
\midrule
Vanilla       & seen   & 66.3 & 0.055 & 0.111 & +50 \\
Vanilla       & unseen & 21.9 & 0.583 & 0.712 & +18 \\
REPA (DINOv3) & seen   & 61.2 & 0.024 & 0.137 & +83 \\
REPA (DINOv3) & unseen & 21.0 & 0.535 & 0.755 & +29 \\
REPA (SigLIP) & seen   & 65.7 & 0.032 & 0.126 & +75 \\
REPA (SigLIP) & unseen & 23.5 & 0.827 & 1.020 & +19 \\
\bottomrule
\end{tabular}
\end{adjustbox}
\end{wraptable}

We also provide a per-feature CP breakdown in \Cref{app:cp-validation}, where gains concentrate on biologically meaningful texture and shape axes: per-cell texture autocorrelation, image-level granularity, total cell-segmentation area, and cytoplasm shape regularity.

\noindent\textbf{Takeaway.} On \rxrx, \trust{}-selected samples are closer to real ones in both calibrated CellProfiler morphology space and under the learned DINOv3 validation encoder. Since the metric accounts for within-condition real-real variation, gains directly reflect reductions in excess morphology deviation beyond normal biological spread.

\subsection{During generation}
\label{sec:during-gen}
In this section, we investigate the extent to which the score \trust{} can be evaluated at generation time before decoding so that low confidence images are never generated, directly mapping diffusion-model representations into the \fext space (\Cref{par:during-generation}).
%
 We train a lightweight \emph{\transl} for this mapping, with an objective that preserves the shared-covariance Mahalanobis geometry the score relies on, allowing internal-feature trust scoring without the VAE decoder or \fext.
\begin{wrapfigure}[15]{r}{0.48\textwidth}
\vspace{-1em}
    \centering
    \includegraphics[width=\linewidth]{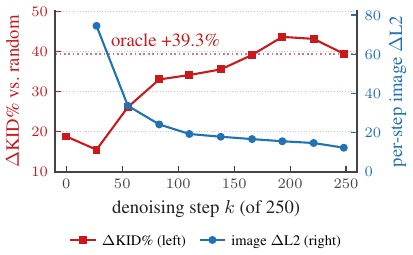}
    \caption{\textbf{\Transl scoring across denoising (\celeba, Vanilla SiT-B/2, 250 steps).} P95-real-threshold $\Delta$KID improvement rises as the predicted-clean trajectory settles.}
    \label{fig:celeba-timestep-fine}
\end{wrapfigure}
\Cref{fig:celeba-timestep-fine} evaluates the \transl at intermediate denoising steps on \celeba held-out (250-step sampler). At $k\!=\!83$ the filter already reaches $33\%$ $\Delta$KID improvement while saving $67\%$ of denoising compute; by $k\!\approx\!166$ it matches the post-generation oracle while still saving $34\%$. \rxrx figure, full-support rows, decile-binning confirmation, and per-step numbers are in \Cref{app:timestep-rxrx1}, \Cref{app:during-gen-full}, \Cref{app:detailed_results}, and \Cref{app:timestep-fine}.

\looseness=-1
\Cref{tab:aligned-combined} reports the support-shift held-out result. The \transl gives $25$--$34\%$ $\Delta$KID improvements on \rxrx with $\rho(\trust)$ up to $0.88$, and matches post-generation filtering on \celeba. Here, REPA-aligned features serve as a natural baseline: REPA training aligns internal states to a pretrained teacher as part of the generator's loss, so any REPA-trained generator provides a \transl at no additional cost. This baseline partly preserves the signal on \celeba ($\rho(\trust)=0.95$, $+31\%$ $\Delta$KID) but fails on \rxrx, with selection worse than random and $\rho(\trust)\!\approx\!0$.Cosine distance, as used in REPA, aligns internal states to the extractor features without preserving the covariance structure needed by the Mahalanobis trust score. Thus, REPA features are not a reliable geometry for trust scoring.

\begin{table}[t]
\centering
\caption{\textbf{Trust scoring during generation}.
Same setup as \Cref{tab:fpr95-selection}
but the score here is evaluated from internal features mapped by a \transl instead of post-generation SigLIP ones. The joint-aligned REPA baseline partly works on \celeba but fails on \rxrx; our \transl restores filtering and ranking on both.}
\label{tab:aligned-combined}
\small
\begin{adjustbox}{width=0.95\textwidth}
\begin{tabular}{cll cccc c}
\rowcolor{gray!20}\toprule
& Feature source & Model & Accept\% & $\Delta$KID$_\text{trust}$$\downarrow$ & $\Delta$KID$_\text{baseline}$$\downarrow$ & $\Delta$\%$\uparrow$ & $\rho(\trust)$$\uparrow$ \\
\midrule
\multirow{5}{*}{\rotatebox[origin=c]{90}{\celeba}}
& REPA  & REPA (DINOv3) & 62.7 & 0.278{\tiny$\pm$.015} & 0.404{\tiny$\pm$.021} & $+$31.2 & 0.95 \\
& REPA  & REPA (SigLIP) & 66.7 & 0.293{\tiny$\pm$.024} & 0.425{\tiny$\pm$.018} & $+$31.1 & 0.95 \\
& \transl & Vanilla       & 23.6 & 0.195{\tiny$\pm$.008} & 0.378{\tiny$\pm$.023} & $+$48.4 & 0.83 \\
& \transl & REPA (DINOv3) & 48.4 & 0.250{\tiny$\pm$.016} & 0.397{\tiny$\pm$.017} & $+$37.2 & 0.92 \\
& \transl & REPA (SigLIP) & 17.2 & 0.308{\tiny$\pm$.015} & 0.409{\tiny$\pm$.015} & $+$24.8 & 0.87 \\
\midrule
\multirow{5}{*}{\rotatebox[origin=c]{90}{\rxrx}}
& REPA  & REPA (DINOv3) &  1.4 & 0.537{\tiny$\pm$.032} & 0.329{\tiny$\pm$.037} & $-$63.4 & 0.09 \\
& REPA  & REPA (SigLIP) &  3.9 & 0.564{\tiny$\pm$.022} & 0.293{\tiny$\pm$.023} & $-$92.3 & 0.13 \\
& \transl & Vanilla       & 38.4 & 0.260{\tiny$\pm$.003} & 0.346{\tiny$\pm$.016} & $+$24.9 & 0.78 \\
& \transl & REPA (DINOv3) & 62.3 & 0.246{\tiny$\pm$.003} & 0.372{\tiny$\pm$.037} & $+$33.8 & 0.88 \\
& \transl & REPA (SigLIP) & 60.5 & 0.235{\tiny$\pm$.011} & 0.329{\tiny$\pm$.018} & $+$28.7 & 0.70 \\
\bottomrule
\end{tabular}
\end{adjustbox}
\end{table}

\noindent\textbf{Takeaway.} Preserving the trust geometry during generation enables early scoring; the same map further allows the score to be evaluated well before final decoding, increasing computational benefits.


\section{Conclusion}

\looseness=-1In this work, we studied per-sample quality assessment for conditional generation under compositional shift, where real samples from the requested target condition are unavailable. We showed that full conditional fidelity to a missing target distribution is not identifiable from the observed training distribution alone, but that global realism and attribute-wise faithfulness remain estimable.
We introduced a post-hoc trust score combining a pooled Mahalanobis realism term with shared-covariance attribute margins. Under reference coverage, reference-anchored attribute comparisons are identifiable. Empirically, the score supports filtering, ranking, and synthetic-data curation, improving KID, downstream utility, and RxRx1 CellProfiler morphology alignment.
The same real-data geometry can also be used during generation by mapping denoising states into the pretrained-encoder feature space, enabling early abstention when the trust geometry is preserved.
A limitation of the current theory is that it focuses on discrete attributes under sufficient coverage conditions. Future work should extend the framework to continuous conditioning variables, characterize weaker identifiability conditions, and use the score to drift generations toward higher-trust regions rather than only abstaining from poor ones.

\section*{Acknowledgements}
This work was supported by the Chan Zuckerberg Initiative (CZI) through its AI Residency Program. We are grateful to CZI for the opportunity to take part in the program, and to the CZI AI Infrastructure Team for providing support with the GPU cluster used to train our models. MF is supported by the MSCA IST-Bridge fellowship which has received funding from the European Union’s Horizon 2020 research and innovation program under the Marie Skłodowska-Curie grant agreement No 101034413.

\bibliographystyle{plain}
\bibliography{main}

\newpage
\appendix
\section{Model Pipeline}
\label{app:model-pipeline}
Figure~\ref{fig:mod_pipeline} contrasts post-generation and during-generation trust scoring.
Post-generation scoring first completes denoising, decodes the final latent into an image, and applies the feature extractor $\Phi$ before evaluating the trust score.
During-generation scoring instead maps an intermediate diffusion representation $h_t$ into the same $\Phi$-compatible space using a learned translator $T$, bypassing decoding and feature extraction.
Thus, both settings use the same calibrated trust geometry, while during-generation scoring enables assessment before full generation is completed.

\begin{figure}[h!]
    \centering
    \includegraphics[width=\linewidth]{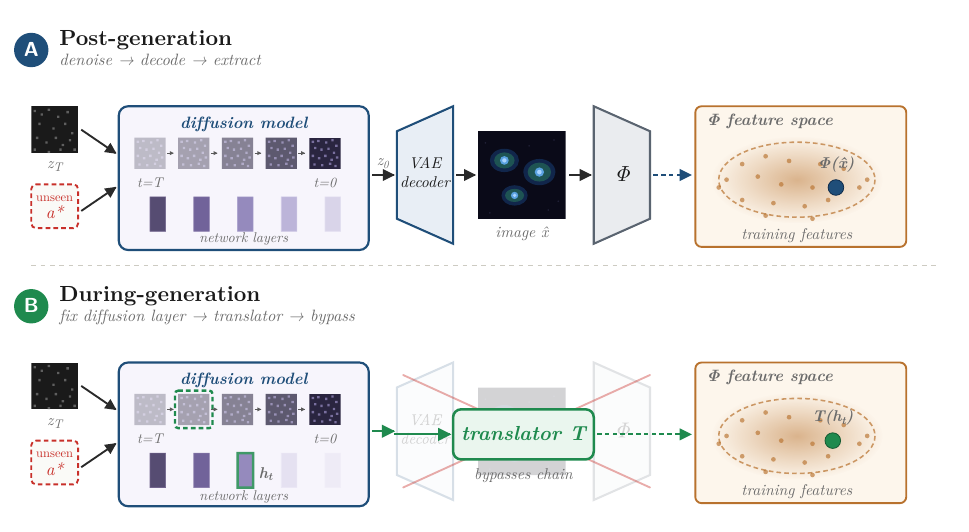}
\caption{
\emph{Post-generation and during-generation trust scoring.}
In the post-generation setting, the diffusion model is run to completion, the final latent is decoded into an image, and the feature extractor $\Phi$ maps the generated sample into the calibrated feature space where the trust score is evaluated.
In the during-generation setting, scoring bypasses decoding and feature extraction: an intermediate denoising representation $h_t$ from a fixed diffusion layer, and selected on a possibly early diffusion timestep $t$,  is mapped by a learned translator $T$ into the same $\Phi$-compatible feature space.
Both pipelines therefore use the same real-data-calibrated trust geometry, but during-generation scoring enables early assessment and abstention before full generation is completed.}
\label{fig:mod_pipeline}
\end{figure}

\section{Additional theoretical details}
\label{app:theory-details}

\subsection{Minimality for unconfounded attribute coverage}
\label{app:reference-coverage-minimal-proof}

\begin{proposition}[Reference coverage is minimal]
\label{prop:reference-coverage-minimal}
Let \(S\subseteq \mathcal A_1\times\cdots\times\mathcal A_K\) be an observed support.
Define the one-attribute contrast graph \(G_S\) to have vertices \(a\in S\), with an
edge between two vertices if they differ in exactly one attribute. Without structural
assumptions on \(\mathbb E[y\mid a]\), suppose \(S\) permits unconfounded identification
of every attribute value effect, in the sense that \(G_S\) contains a connected subgraph
whose vertices cover every value in each \(\mathcal A_k\). Then
\[
|S|
\ge
1+\sum_{k=1}^K(|\mathcal A_k|-1)
=
N_{\mathrm{ref}}.
\]
Reference coverage attains this bound, since it contains exactly \(N_{\mathrm{ref}}\)
joint conditionings. Hence reference coverage is cardinality-minimal.
\end{proposition}

\begin{proof}[Proof of \Cref{prop:reference-coverage-minimal}]
Let
\[
m(a):=\mathbb E[y\mid a]
\]
denote the conditional mean at attribute configuration \(a\). With no structural
assumptions, \(m(a)\) is an arbitrary function of the full joint attribute
configuration.

A model-free contrast can be attributed to a single attribute only when the two
conditions being compared differ in that attribute alone. If two observed
conditions differ in multiple attributes, then the resulting mean difference can
be decomposed in multiple ways among the changed attributes and their
interactions. Thus, without structural assumptions, such a contrast is support
confounded.

Define the one-attribute contrast graph \(G_S\) of an observed support
\(S\subseteq \mathcal A_1\times\cdots\times\mathcal A_K\) as follows. The
vertices are the observed joint conditions \(a\in S\). Two vertices are connected
by an edge labeled \(k\) if they differ only in coordinate \(k\). Such an edge
gives an unconfounded contrast for attribute \(k\).

Suppose an observed support identifies unconfounded effects for all attribute
values. Then its one-attribute contrast graph must contain a connected subgraph
whose vertices cover every attribute value: for every attribute \(k\) and every
value \(v\in\mathcal A_k\), some vertex in the subgraph must satisfy \(a_k=v\).
Otherwise, some value would either be unobserved in any one-attribute comparison
or would be observed only through joint changes with other attributes, leaving
its individual contribution non-identifiable.

Now consider any connected subgraph \(T\) of \(G_S\) that covers all attribute
values. Let \(R\) be a spanning tree of \(T\), and choose an arbitrary root
condition \(a^{(0)}\). The root covers exactly one value of each attribute, hence
\(K\) attribute values in total. Each additional vertex in the tree differs from
its parent in exactly one coordinate, so it can introduce at most one previously
unseen attribute value. The total number of attribute values that must be covered
is
\[
\sum_{k=1}^K |\mathcal A_k|.
\]
After the root, the number of remaining attribute values that must be introduced
is therefore
\[
\sum_{k=1}^K |\mathcal A_k| - K
=
\sum_{k=1}^K (|\mathcal A_k|-1).
\]
Since each new observed joint condition can introduce at most one such value,
any support that identifies all attribute values through unconfounded
one-attribute contrasts must contain at least
\[
1+\sum_{k=1}^K (|\mathcal A_k|-1)
=
N_{\mathrm{ref}}
\]
observed joint conditions.

Reference coverage attains this lower bound. Indeed, for a reference condition
\(\bar a\), it observes the baseline condition \(\bar a\) and, for every
attribute \(k\) and every non-reference value \(v\neq \bar a_k\), the
one-attribute deviation
\[
(v,\bar a_{-k}).
\]
Thus it contains exactly
\[
1+\sum_{k=1}^K (|\mathcal A_k|-1)
=
N_{\mathrm{ref}}
\]
joint conditions. These conditions form a star-shaped one-attribute contrast
graph centered at \(\bar a\), and the graph covers every attribute value.
Therefore reference coverage is cardinality-minimal.

This minimal design is not unique: other supports with \(N_{\mathrm{ref}}\)
conditions may also cover all attribute values through unconfounded
one-attribute contrasts. Reference coverage is the shared-context construction
that attains the minimum.
\end{proof}

\subsection{Support confounding without reference coverage}

\label{app:support-confounding}
\begin{proposition}[Support confounding destroys attribute attribution]
\label{prop:support-confounding}
There exist observed supports for which attribute-specific contributions are not determined by the observed distribution. In particular, with four binary attributes and observed support
\[
S=\{0000,1111\},
\]
the observed data distribution determines only the joint contrast between \(0000\) and \(1111\); it does not determine any unique decomposition of that contrast into attribute-specific contributions.
\end{proposition}
\begin{proof}[Proof of \Cref{prop:support-confounding}]
It suffices to consider a scalar feature coordinate. Let the observed support be
\[
S=\{0000,1111\},
\]
and write
\[
m_0:=\mathbb{E}[y\mid a=0000],
\qquad
m_1:=\mathbb{E}[y\mid a=1111].
\]
The observed data distribution determines \(m_0\) and \(m_1\), hence only the joint difference \(m_1-m_0\).

Now consider an additive representation on the observed support,
\[
\mathbb{E}[y\mid a]
=
\beta_0+\sum_{j=1}^4 g_j(a_j).
\]
On the two observed cells, this implies
\[
m_1-m_0
=
\sum_{j=1}^4 \big(g_j(1)-g_j(0)\big).
\]
Only the sum is determined; the individual differences are not.

For example, the following two choices agree on the observed support but assign different attribute-1 contributions.

\paragraph{Model A.}
Set
\[
g_1(1)-g_1(0)=m_1-m_0,
\qquad
g_j(1)-g_j(0)=0 \quad (j=2,3,4).
\]

\paragraph{Model B.}
Set
\[
g_j(1)-g_j(0)=\frac{m_1-m_0}{4}
\qquad (j=1,2,3,4).
\]

Both models reproduce exactly the same observed conditional means at \(0000\) and \(1111\), but they attribute different portions of the observed change to attribute \(1\). Hence the observed distribution on \(S\) does not determine any unique attribute-specific contribution. This is the support confounding claimed in the proposition.
\end{proof}

\subsection{Identification and well-definedness of anchored objects}
\label{app:well-definedness}

We collect here the basic support-side consequences used in the main text.

\begin{proof}[Proof of \Cref{prop:anchored-identification}]
Fix an attribute \(k\) and values \(u,v\in\mathcal{A}_k\). Under reference coverage,
\[
\mathbb{P}(a_k=v,\;a_{-k}=\bar a_{-k})>0
\qquad\text{and}\qquad
\mathbb{P}(a_k=u,\;a_{-k}=\bar a_{-k})>0.
\]
Hence the conditional expectations
\[
\mu^{\mathrm{ref}}_{k,v}
=
\mathbb{E}[y\mid a_k=v,\;a_{-k}=\bar a_{-k}]
\qquad\text{and}\qquad
\mu^{\mathrm{ref}}_{k,u}
=
\mathbb{E}[y\mid a_k=u,\;a_{-k}=\bar a_{-k}]
\]
are functions of the observed joint distribution of \((y,a)\) on positive-probability events, and are therefore point-identified whenever \(y\) is integrable. Their difference
\[
\Delta^{\mathrm{ref}}_{k}(v,u)=\mu^{\mathrm{ref}}_{k,v}-\mu^{\mathrm{ref}}_{k,u}
\]
is then point-identified as well. Since \(P_k\) is fixed, the oracle distances are deterministic functions of the identified prototypes. The oracle margin is a finite minimum over identified distances, and is therefore also identified.
\end{proof}

\begin{proposition}[Well-definedness under reference coverage]
\label{prop:well-definedness}
Assume the support has reference coverage with reference condition \(\bar a\). Then for every attribute \(k\in\{1,\dots,K\}\) and every value \(v\in\mathcal{A}_k\),
\[
\mathbb{P}(a_k=v,\;a_{-k}=\bar a_{-k})>0
\qquad\text{and}\qquad
\mathbb{P}(a_k=v)>0.
\]
Consequently, whenever \(y\) is integrable, both
\[
\mu^{\mathrm{ref}}_{k,v}:=\mathbb{E}[y\mid a_k=v,\;a_{-k}=\bar a_{-k}]
\qquad\text{and}\qquad
\eta_{k,v}:=\mathbb{E}[y\mid a_k=v]
\]
are well-defined. If in addition \(P_k\succ 0\), then for every \(y\in\mathbb{R}^d\),
\[
d_k^{\mathrm{ref}}(y;v):=(y-\mu^{\mathrm{ref}}_{k,v})^\top P_k (y-\mu^{\mathrm{ref}}_{k,v}),
\qquad
d_k(y;v):=(y-\eta_{k,v})^\top P_k (y-\eta_{k,v})
\]
and the corresponding margins
\[
M_k^{\mathrm{ref}}(y;t):=d_k^{\mathrm{ref}}(y;t)-\min_{v\neq t}d_k^{\mathrm{ref}}(y;v),
\qquad
M_k(y;t):=d_k(y;t)-\min_{v\neq t}d_k(y;v)
\]
are well-defined for all targets \(t\in\mathcal{A}_k\).
\end{proposition}

\begin{proof}
Reference coverage states directly that
\[
\mathbb{P}(a_k=v,\;a_{-k}=\bar a_{-k})>0
\]
for every \(k\) and \(v\). Marginalizing over \(a_{-k}\) then gives
\[
\mathbb{P}(a_k=v)\geq \mathbb{P}(a_k=v,\;a_{-k}=\bar a_{-k})>0.
\]
Hence both conditioning events have positive probability. Since \(y\) is integrable, the conditional expectations defining \(\mu^{\mathrm{ref}}_{k,v}\) and \(\eta_{k,v}\) are well-defined. If \(P_k\succ 0\), then the quadratic forms \(d_k^{\mathrm{ref}}(y;v)\) and \(d_k(y;v)\) are finite for every \(y\in\mathbb{R}^d\), and since \(\mathcal{A}_k\) is finite, the minima over \(v\neq t\) are attained and finite. Therefore the margins are well-defined.
\end{proof}

\begin{remark}[Identification versus estimation]
Reference coverage is a support-side condition ensuring that the fixed-context comparison
\(\mu^{\mathrm{ref}}_{k,v}\) is identified and well-defined. It does not by itself guarantee low-variance estimation in finite samples. Finite-sample issues enter through how many real samples populate each required slice, and are separate from the population-level argument.
\end{remark}

\subsection{Pooled and reference prototypes}
\label{app:pooled-reference-prototypes}
The main identifiability result concerns the reference-anchored prototypes
\(\mu^{\mathrm{ref}}_{k,v}\). The implemented score uses pooled prototypes
\(\eta_{k,v}\) because they are more sample-efficient: they use all real samples with \(a_k=v\), rather than only samples in the anchored cell
\((a_k=v,a_{-k}=\bar a_{-k})\). This appendix gives sufficient conditions under which this lower-variance pooled comparator agrees with the reference-anchored comparator.

\paragraph{Conditional-mean decomposition and pooling assumptions.}

The pooled prototype for attribute \(k\) and value \(v\) is
\[
\eta_{k,v}:=\mathbb{E}[y\mid a_k=v].
\]
For the pooled-reference analysis, write the conditional mean as
\[
\mathbb{E}[y\mid a_k=v,\;a_{-k}=b]
=
s_{k,v}+c_k(b)+\delta_{k,v}(b),
\]
where \(s_{k,v}\) is the value-associated signal, \(c_k(b)\) is the remaining-context contribution, and \(\delta_{k,v}(b)\) is a value-dependent interaction residual. This decomposition is only an analysis device; the score does not require estimating these terms.

We use the following two assumptions:
\[
\tag{A1}
\mathbb{E}[c_k(a_{-k})\mid a_k=v]=\bar c_k
\qquad
\text{for all }v\in\mathcal{A}_k,
\]
and
\[
\tag{A2}
\mathbb{E}[\delta_{k,v}(a_{-k})\mid a_k=v]=0
\qquad
\text{for all }v\in\mathcal{A}_k.
\]
A1 is the substantive context-balance condition: after pooling over observed contexts, the average non-\(k\) context contribution does not depend on the value of attribute \(k\). A2 is a centering convention rather than an additional support assumption: any value-specific average interaction bias can be absorbed into \(s_{k,v}\).

Taking expectation over \(a_{-k}\mid a_k=v\), we obtain
\[
\eta_{k,v}
=
s_{k,v}
+
\mathbb{E}[c_k(a_{-k})\mid a_k=v]
+
\mathbb{E}[\delta_{k,v}(a_{-k})\mid a_k=v].
\]
Under A1--A2, this reduces to
\[
\eta_{k,v}=s_{k,v}+\bar c_k.
\]
At the reference context,
\[
\mu^{\mathrm{ref}}_{k,v}
=
s_{k,v}
+
c_k(\bar a_{-k})
+
\delta_{k,v}(\bar a_{-k}).
\]
Thus pooled prototypes share the same value signal \(s_{k,v}\) as the reference prototypes, but differ by a shared context shift and a value-dependent reference leakage term. The next lemma makes this relation explicit.

\begin{lemma}[Pooled-reference decomposition]
\label{lem:pooled-reference-decomposition}
Under A1--A2,
\[
\eta_{k,v}
=
\mu^{\mathrm{ref}}_{k,v}
+
\gamma_k^{\mathrm{ref}}
-
e_{k,v}^{\mathrm{ref}},
\]
where
\[
\gamma_k^{\mathrm{ref}}=\bar c_k-c_k(\bar a_{-k}),
\qquad
e_{k,v}^{\mathrm{ref}}=\delta_{k,v}(\bar a_{-k}).
\]
\end{lemma}

\begin{proof}
By the conditional-mean decomposition,
\[
\eta_{k,v}
=
s_{k,v}
+
\mathbb{E}[c_k(a_{-k})\mid a_k=v]
+
\mathbb{E}[\delta_{k,v}(a_{-k})\mid a_k=v].
\]
Using A1--A2 gives
\[
\eta_{k,v}=s_{k,v}+\bar c_k.
\]
At the reference context,
\[
\mu^{\mathrm{ref}}_{k,v}
=
s_{k,v}+c_k(\bar a_{-k})+\delta_{k,v}(\bar a_{-k}).
\]
Subtracting yields
\[
\eta_{k,v}
=
\mu^{\mathrm{ref}}_{k,v}
+
\bar c_k-c_k(\bar a_{-k})
-
\delta_{k,v}(\bar a_{-k}).
\]
\end{proof}

\subsection{Pooled-reference perturbation bound}
\label{app:pooled-reference-proof}

\begin{proposition}[Pooled margin perturbation, value-specific split form]
\label{prop:pooled-margin-perturbation}
Assume A1--A2 and let
\[
\tilde y=y-\gamma_k^{\mathrm{ref}}.
\]
Suppose that for each \(v\in\mathcal A_k\),
\[
\|e_{k,v}^{\mathrm{ref}}\|_{P_k}
\leq
\varepsilon_{k,v}^{\mathrm{ref}}.
\]
Define
\[
\rho_{k,v}^{\mathrm{ref}}(y)
:=
\|\tilde y-\mu_{k,v}^{\mathrm{ref}}\|_{P_k},
\qquad
b_{k,v}^{\mathrm{ref}}(y)
:=
2\varepsilon_{k,v}^{\mathrm{ref}}
\rho_{k,v}^{\mathrm{ref}}(y)
+
(\varepsilon_{k,v}^{\mathrm{ref}})^2.
\]
For a target value \(t\), define
\[
B_{k,\mathrm{split}}^{\mathrm{ref}}(y;t)
:=
b_{k,t}^{\mathrm{ref}}(y)
+
\max_{v\neq t}b_{k,v}^{\mathrm{ref}}(y).
\]
Then
\[
\big|
M_k(y;t)-M_k^{\mathrm{ref}}(\tilde y;t)
\big|
\leq
B_{k,\mathrm{split}}^{\mathrm{ref}}(y;t).
\]
\end{proposition}

\begin{proof}
By \Cref{lem:pooled-reference-decomposition},
\[
y-\eta_{k,v}
=
\tilde y-\mu^{\mathrm{ref}}_{k,v}+e_{k,v}^{\mathrm{ref}}.
\]
Therefore
\[
d_k(y;v)
=
\|\tilde y-\mu^{\mathrm{ref}}_{k,v}+e_{k,v}^{\mathrm{ref}}\|_{P_k}^2,
\]
while
\[
d_k^{\mathrm{ref}}(\tilde y;v)
=
\|\tilde y-\mu^{\mathrm{ref}}_{k,v}\|_{P_k}^2.
\]
Define
\[
\Delta_v
:=
d_k(y;v)-d_k^{\mathrm{ref}}(\tilde y;v).
\]
Expanding the square gives
\[
\Delta_v
=
2\left\langle
\tilde y-\mu_{k,v}^{\mathrm{ref}},
e_{k,v}^{\mathrm{ref}}
\right\rangle_{P_k}
+
\|e_{k,v}^{\mathrm{ref}}\|_{P_k}^2.
\]
By Cauchy--Schwarz and the value-specific leakage bound,
\[
|\Delta_v|
\leq
2\varepsilon_{k,v}^{\mathrm{ref}}
\|\tilde y-\mu_{k,v}^{\mathrm{ref}}\|_{P_k}
+
(\varepsilon_{k,v}^{\mathrm{ref}})^2
=
b_{k,v}^{\mathrm{ref}}(y).
\]
Therefore
\[
|\Delta_t|\leq b_{k,t}^{\mathrm{ref}}(y).
\]
For the competitor term, using
\[
\left|
\min_{v\neq t} a_v-\min_{v\neq t} b_v
\right|
\leq
\max_{v\neq t}|a_v-b_v|,
\]
we obtain
\[
\left|
\min_{v\neq t}d_k(y;v)
-
\min_{v\neq t}d_k^{\mathrm{ref}}(\tilde y;v)
\right|
\leq
\max_{v\neq t}b_{k,v}^{\mathrm{ref}}(y).
\]
Adding target and competitor errors gives
\[
\big|
M_k(y;t)-M_k^{\mathrm{ref}}(\tilde y;t)
\big|
\leq
b_{k,t}^{\mathrm{ref}}(y)
+
\max_{v\neq t}b_{k,v}^{\mathrm{ref}}(y)
=
B_{k,\mathrm{split}}^{\mathrm{ref}}(y;t).
\]

\end{proof}

\begin{corollary}[Buffered pooled-comparator decision equivalence]
\label{cor:comparator-pooled-decision}
Under the assumptions of \Cref{prop:pooled-margin-perturbation},
\[
\left|M_k^{\mathrm{ref}}(\tilde y;t)\right|
>
B_{k,\mathrm{split}}^{\mathrm{ref}}(y;t)
\quad\Longrightarrow\quad
\operatorname{sign} M_k(y;t)
=
\operatorname{sign} M_k^{\mathrm{ref}}(\tilde y;t).
\]
The same conclusion holds with \(M_k(y;t)\) and
\(M_k^{\mathrm{ref}}(\tilde y;t)\) interchanged.
\end{corollary}

\begin{proof}
\Cref{prop:pooled-margin-perturbation} gives
\[
\left|
M_k(y;t)-M_k^{\mathrm{ref}}(\tilde y;t)
\right|
\leq
B_{k,\mathrm{split}}^{\mathrm{ref}}(y; t).
\]
Thus if either margin has magnitude larger than \(B_{k,\mathrm{split}}^{\mathrm{ref}}(y;t)\),
the other margin has the same sign. This proves the claim.
\end{proof}

\subsection{Empirical decision agreement (\Cref{cor:comparator-pooled-decision})}
\label{app:perturbation-bound-empirical}

The value-specific split certificate is
\[
B_{k,\mathrm{split}}^{\mathrm{ref}}(y;t)
=
b_{k,t}(y)
+
\max_{v\neq t} b_{k,v}(y),
\qquad
b_{k,v}(y)
=
2\varepsilon_{k,v}^{\mathrm{ref}}
\|\tilde y-\mu_{k,v}^{\mathrm{ref}}\|_{P_k}
+
(\varepsilon_{k,v}^{\mathrm{ref}})^2 .
\]
When
\[
|M_k^{\mathrm{ref}}(\tilde y;t)|
>
B_{k,\mathrm{split}}^{\mathrm{ref}}(y;t),
\]
\Cref{cor:comparator-pooled-decision} certifies that the pooled and reference-anchored margin signs agree.

We instantiate this diagnostic on held-out real samples from both datasets, where the reference-anchored comparator can be evaluated. For each attribute \(k\), we estimate \(\eta_{k,v}\), \(P_k\), and \(\mu_{k,v}^{\mathrm{ref}}\) using the same calibration protocol as the trust-score experiments, and evaluate on the canonical held-out evaluation splits.

For each sample, let \(a_k^\star(y)\) denote the target value. We compare the deployed pooled margin
\[
M_k(y;a_k^\star)
=
d_k(y;a_k^\star)
-
\min_{v\neq a_k^\star} d_k(y;v)
\]
with the corresponding reference-anchored margin
\[
M_k^{\mathrm{ref}}(\tilde y;a_k^\star)
=
d_k^{\mathrm{ref}}(\tilde y;a_k^\star)
-
\min_{v\neq a_k^\star} d_k^{\mathrm{ref}}(\tilde y;v).
\]
The certificate used in \Cref{tab:pooled-reference-certification} is
\[
B_{k,\mathrm{split}}^{\mathrm{ref}}(y;a_k^\star)
=
b_{k,a_k^\star}(y)
+
\max_{v\neq a_k^\star} b_{k,v}(y).
\]

\begin{table}[t]
\caption{
Empirical pooled/reference sign certification using the split bound
$
B_{k,\mathrm{split}}^{\mathrm{ref}}(y;t)
=
b_{k,t}(y)+\max_{v\neq t}b_{k,v}(y),
\qquad
b_{k,v}(y)
=
2\varepsilon_{k,v}^{\mathrm{ref}}
\|\tilde y-\mu_{k,v}^{\mathrm{ref}}\|_{P_k}
+
(\varepsilon_{k,v}^{\mathrm{ref}})^2.
$
$\mathrm{med}\,B_{\mathrm{split}}$ is the median certificate radius, and
$\Pr[\mathrm{certified}]$ is the fraction of samples satisfying
$|M_k^{\mathrm{ref}}(\tilde y;t)|>B_{k,\mathrm{split}}^{\mathrm{ref}}(y;t)$.
Agree(all) and Agree(cert.) report sign agreement before and after certification.
}
\centering
\small
\setlength{\tabcolsep}{4pt}
\begin{adjustbox}{width=\linewidth}
\begin{tabular}{llrrrrrrrr}
\rowcolor{gray!20}\toprule
Dataset
& Attribute
& $N_{\mathrm{eval}}$
& $V$
& $\varepsilon_k^{\mathrm{ref}}$
& $|M^{\mathrm{ref}}|^{\mathrm{med}}$
& $\mathrm{med}\,B_{\mathrm{split}}$
& $\Pr[\mathrm{certified}]$
& Agree(all)
& Agree(cert.) \\
\midrule
CelebA & Male
& 162770 & 2 & 0.79 & 74.2 & 53.0 & \textbf{0.862} & \textbf{1.000} & \textbf{1.000} \\
CelebA & Smiling
& 162770 & 2 & 0.70 & 22.6 & 44.8 & 0.007 & 0.998 & \textbf{1.000} \\
CelebA & Blond\_Hair
& 162770 & 2 & 0.78 & 21.6 & 49.7 & 0.008 & 0.994 & \textbf{1.000} \\
CelebA & Eyeglasses
& 162770 & 2 & 0.73 & 13.3 & 46.0 & \(<0.001\) & 0.990 & \textbf{1.000} \\
RxRx1 & cell\_type\_id
& 758 & 4 & 0.79 & 14.0 & 15.2 & 0.421 & 0.997 & \textbf{1.000} \\
RxRx1 & sirna\_id
& 640 & 20 & 1.04 & 1.2 & 16.6 & 0.000 & 0.803 & --- \\
\bottomrule
\end{tabular}
\end{adjustbox}

\label{tab:pooled-reference-certification}
\end{table}

\paragraph{Certified agreement.}
Across every row with nonzero certified support, the certified agreement rate is \(1.000\). This is the empirical behavior predicted by \Cref{cor:comparator-pooled-decision}: whenever the reference margin exceeds the split-radius perturbation band, the pooled and reference-anchored decisions have matching signs. The certificate is strongest for CelebA Male, where \(86.2\%\) of samples are certified, and for RxRx1 cell type, where \(42.1\%\) are certified.

\paragraph{Agreement beyond the certified regime.}
The certificate is sufficient, not necessary. Thus low certified coverage does not imply poor pooled/reference agreement. On CelebA, agreement remains very high for all attributes, ranging from \(0.990\) to \(1.000\), even when the certified fraction is small. RxRx1 cell type also has high agreement, \(0.997\), with substantial certified coverage. The main failure case is the finer-grained RxRx1 sirna decision: the median reference margin is only \(1.2\), while the median split bound is \(16.6\), so no samples are certified and overall agreement drops to \(0.803\).

\paragraph{Why certification is hard in low-margin regimes.}
The split bound is tighter than the earlier one-radius bound, but it can still be conservative when the reference margin is small relative to the leakage-weighted prototype radii. This is most visible for Smiling, Blond\_Hair, and Eyeglasses, where empirical agreement is high but few samples have margins large enough to clear the certificate. In the RxRx1 sirna setting, the issue is more severe: many sirna prototypes are close in the reference geometry, so the deployed margin is often small compared with the perturbation band. This is exactly the regime where the theory predicts pooled scoring should be least reference-anchored comparator-like.

\subsection{Empirical A1 context-balance diagnostic}
\label{app:a1-diagnostic}

A1 ($\mathbb{E}[c_k(a_{-k})\mid a_k=v]=\bar c_k$) is the load-bearing condition that lets pooled scoring stand in for the reference-anchored comparator. By \Cref{lem:pooled-reference-decomposition}, under A1+A2 the difference $\eta_{k,v}-\mu_{k,v}^{\mathrm{ref}}=\gamma_k^{\mathrm{ref}}-e_{k,v}^{\mathrm{ref}}$ is $v$-independent up to the leakage $e_{k,v}^{\mathrm{ref}}=\delta_{k,v}(\bar a_{-k})$. We therefore use $\varepsilon_k^{\mathrm{ref}}=\max_v\|e_{k,v}^{\mathrm{ref}}\|_{P_k}$ as an operational A1-violation scale and compare it to the inter-prototype reference scale $\|\mu_{k,u}^{\mathrm{ref}}-\mu_{k,t}^{\mathrm{ref}}\|_{P_k}$, which estimates the value-signal gap $\|s_{k,u}-s_{k,t}\|_{P_k}$ in \Cref{lem:pooled-reference-decomposition}. A ratio $\ll 1$ means that any A1 violation is small relative to the attribute signal in the reference geometry; a ratio $\gtrsim 1$ indicates that leakage is large enough to obscure at least one value-pair comparison. Estimation reuses exactly the fitting and evaluation pools of \Cref{tab:pooled-reference-certification}; \Cref{tab:a1-diagnostic} reports the resulting per-attribute leakage-to-gap ratios.

 \begin{table}[h]
  \centering
  \caption{Empirical A1 diagnostic. ``min gap'' and ``median gap''
  are the minimum and median pairwise $P_k$-Mahalanobis distances $\|\mu_{k,u}^{\mathrm{ref}}-\mu_{k,t}^{\mathrm{ref}}\|_{P_k}$ over $u\neq t$;
  Smaller $\varepsilon/\text{gap}$ indicates greater
  compatibility with the A1 regime. Bold flags the regime where leakage exceeds the closest value-pair gap.}
  \label{tab:a1-diagnostic}
  \small
  \begin{tabular}{ll r r r r r r}
  \rowcolor{gray!20}\toprule
  Dataset & Attribute & $V$ & $\varepsilon_k^{\mathrm{ref}}$ & min gap & median gap & $\varepsilon/$min & $\varepsilon/$median \\
  \midrule
  CelebA & Male        & 2  & 0.79 & 8.81 & 8.81 & 0.090 & 0.090 \\
  CelebA & Smiling     & 2  & 0.70 & 4.99 & 4.99 & 0.141 & 0.141 \\
  CelebA & Blond\_Hair & 2  & 0.78 & 4.86 & 4.86 & 0.160 & 0.160 \\
  CelebA & Eyeglasses  & 2  & 0.73 & 3.84 & 3.84 & 0.190 & 0.190 \\
  \midrule
  RxRx1 & cell\_type\_id & 4  & 0.79 & 3.96 & 5.28 & 0.199 & 0.149 \\
  RxRx1 & sirna\_id      & 20 & 1.04 & 0.90 & 1.64 & \textbf{1.157} & 0.635 \\
  \bottomrule
  \end{tabular}
  \end{table}

\paragraph{Small-leakage regimes.}
On all four CelebA attributes, the leakage-to-gap ratio is below \(0.20\). RxRx1 cell type is also in a small-leakage regime, with \(\varepsilon/\text{median gap}=0.149\) and \(\varepsilon/\text{min gap}=0.199\). These are precisely the settings where empirical pooled-vs-comparator sign agreement is near-perfect: agreement ranges from \(0.990\) to \(1.000\) on CelebA and is \(0.997\) on RxRx1 cell type (\Cref{tab:pooled-reference-certification}).

\paragraph{Large-leakage regime.}
The only regime where leakage exceeds the closest value-pair gap is RxRx1 sirna, with
\(\varepsilon/\text{min gap}=1.157\). Its median leakage-to-gap ratio remains below one,
\(\varepsilon/\text{median gap}=0.635\), but the closest sirna prototypes are sufficiently close that the perturbation band dominates many margins. This matches \Cref{tab:pooled-reference-certification}: RxRx1 sirna has median reference margin \(1.2\), median split certificate radius \(16.6\), no certified samples, and substantially lower pooled/reference agreement \(0.803\).

\paragraph{Takeaway.}
The A1 diagnostic aligns with the empirical certification results. Small leakage-to-gap ratios coincide with high pooled/reference agreement and, for CelebA Male and RxRx1 cell type, substantial certified coverage. The only ratio exceeding one occurs for the closest RxRx1 sirna pair, exactly where certification fails and empirical agreement is weakest. Thus A1 is not merely a formal assumption: its violation scale is measurable and predicts when pooled scoring is a reliable reference-comparator stand-in.

\subsection{Exact translation equivalence}
\label{app:exact-translation}
\begin{corollary}[Exact translation equivalence]
\label{cor:exact-translation}
Suppose
\[
\eta_{k,v}
=
\mu_{k,v}^{\mathrm{ref}}
+
\gamma
\qquad
\text{for all }v\in\mathcal{A}_k.
\]
Then, for every \(y\) and every target \(t\),
\[
M_k(y;t)
=
M_k^{\mathrm{ref}}(y-\gamma;t).
\]
\end{corollary}

\begin{proof}
For every \(v\),
\[
y-\eta_{k,v}
=
y-\gamma-\mu_{k,v}^{\mathrm{ref}}.
\]
Therefore
\[
d_k(y;v)
=
d_k^{\mathrm{ref}}(y-\gamma;v)
\]
for every \(v\). Taking the target distance minus the minimum competitor distance gives the result.
\end{proof}

\subsection{Full-feature Mahalanobis discriminants}
\label{app:full-feature-discriminants}

\begin{proposition}[Pairwise decisions depend only on discriminant directions]
\label{prop:discriminant-directions}
Fix an attribute \(k\), a positive definite matrix \(P_k\), and prototypes
\[
\{m_{k,v}:v\in\mathcal{A}_k\}\subset\mathbb{R}^d.
\]
Define
\[
d_k^{(m)}(y;v)
:=
(y-m_{k,v})^\top P_k (y-m_{k,v}).
\]
For any two values \(t,u\in\mathcal{A}_k\),
\[
d_k^{(m)}(y;t)-d_k^{(m)}(y;u)
=
2(m_{k,u}-m_{k,t})^\top P_k y
+
m_{k,t}^\top P_k m_{k,t}
-
m_{k,u}^\top P_k m_{k,u}.
\]
Thus the sign of the pairwise preference between \(t\) and \(u\) depends on \(y\) only through the scalar projection
\[
\left(P_k(m_{k,u}-m_{k,t})\right)^\top y.
\]
Equivalently, only the component of \(y\) along the span of the pairwise discriminant directions
\[
\operatorname{span}
\left\{
P_k(m_{k,u}-m_{k,t})
:
u\neq t
\right\}
\]
can affect the target-vs-competitor comparison.
\end{proposition}
\begin{proof}[Proof of \Cref{prop:discriminant-directions}]
Fix an attribute \(k\), prototypes
\[
\{m_{k,v}:v\in\mathcal{A}_k\},
\]
and a positive definite matrix \(P_k\). For any two values \(t,u\in\mathcal{A}_k\), expand the two Mahalanobis distances:
\[
d_k^{(m)}(y;t)
=
(y-m_{k,t})^\top P_k (y-m_{k,t})
=
y^\top P_k y
-
2m_{k,t}^\top P_k y
+
m_{k,t}^\top P_k m_{k,t},
\]
and
\[
d_k^{(m)}(y;u)
=
(y-m_{k,u})^\top P_k (y-m_{k,u})
=
y^\top P_k y
-
2m_{k,u}^\top P_k y
+
m_{k,u}^\top P_k m_{k,u}.
\]
Subtracting cancels the common quadratic term \(y^\top P_k y\), giving
\[
d_k^{(m)}(y;t)-d_k^{(m)}(y;u)
=
2(m_{k,u}-m_{k,t})^\top P_k y
+
m_{k,t}^\top P_k m_{k,t}
-
m_{k,u}^\top P_k m_{k,u}.
\]
Thus the pairwise decision between \(t\) and \(u\) depends on \(y\) only through
\[
\left(P_k(m_{k,u}-m_{k,t})\right)^\top y,
\]
using symmetry of \(P_k\). Therefore components of \(y\) orthogonal to the span of the discriminant directions
\[
\operatorname{span}
\left\{
P_k(m_{k,u}-m_{k,t})
:
u,t\in\mathcal{A}_k,\; u\neq t
\right\}
\]
cannot affect any pairwise Mahalanobis decision for attribute \(k\). This proves the proposition.
\end{proof}

\subsection{Robustness details}
\label{app:robustness-details}

This subsection records an additional sufficient condition under which residual context variation does not change the attribute-wise Mahalanobis decision. It is not needed for the pooled-reference identification result above, but helps explain why full-feature scoring can remain stable when residual variation is small in the relevant discriminant directions.

\begin{proposition}[Robustness to residual variation in discriminant directions]
\label{prop:interaction-robustness}
Fix an attribute \(k\), a target value \(t\in\mathcal{A}_k\), and suppose the ideal pooled prototypes satisfy
\[
\eta_{k,v}=\bar c_k+s_{k,v}.
\]
Let a generated sample targeted at value \(t\) have normalized feature
\[
y=\bar c_k+s_{k,t}+\lambda,
\]
where \(\lambda\in\mathbb{R}^d\) collects residual context dependence, imperfect disentanglement, and noise. Then for every competitor \(u\neq t\),
\[
d_k(y;t)-d_k(y;u)
=
2(s_{k,u}-s_{k,t})^\top P_k \lambda
-
\|s_{k,u}-s_{k,t}\|_{P_k}^2,
\]
where \(\|z\|_{P_k}^2:=z^\top P_k z\). Consequently, if
\[
2\max_{u\neq t}\left|(s_{k,u}-s_{k,t})^\top P_k\lambda\right|
<
\min_{u\neq t}\|s_{k,u}-s_{k,t}\|_{P_k}^2,
\]
then
\[
M_k(y;t)<0.
\]
\end{proposition}

\begin{proof}
Since \(\eta_{k,t}=\bar c_k+s_{k,t}\), we have
\[
y-\eta_{k,t}
=
\bar c_k+s_{k,t}+\lambda-(\bar c_k+s_{k,t})
=
\lambda,
\]
and therefore
\[
d_k(y;t)=\lambda^\top P_k \lambda.
\]
For a competitor \(u\neq t\),
\[
y-\eta_{k,u}
=
\bar c_k+s_{k,t}+\lambda-(\bar c_k+s_{k,u})
=
s_{k,t}+\lambda-s_{k,u}.
\]
Hence
\begin{align*}
d_k(y;u)
&=
(s_{k,t}+\lambda-s_{k,u})^\top P_k (s_{k,t}+\lambda-s_{k,u})\\
&=
\lambda^\top P_k\lambda
-2(s_{k,u}-s_{k,t})^\top P_k\lambda
+\|s_{k,u}-s_{k,t}\|_{P_k}^2.
\end{align*}
Subtracting gives
\[
d_k(y;t)-d_k(y;u)
=
2(s_{k,u}-s_{k,t})^\top P_k\lambda
-\|s_{k,u}-s_{k,t}\|_{P_k}^2.
\]
If
\[
2\left|(s_{k,u}-s_{k,t})^\top P_k\lambda\right|
<
\|s_{k,u}-s_{k,t}\|_{P_k}^2
\]
for every \(u\neq t\), then each pairwise difference \(d_k(y;t)-d_k(y;u)\) is strictly negative. Therefore the target value \(t\) beats every competitor under the Mahalanobis comparison, so
\[
M_k(y;t)=d_k(y;t)-\min_{u\neq t}d_k(y;u)<0.
\]
\end{proof}

\section{Estimation details}
\label{app:estimation-details}

This appendix records the finite-sample estimators used to instantiate the population-level objects from the main text.

\subsection{Calibration details for score normalization constants}

Let \(\mathcal{D}_{\mathrm{real}}\) denote the real-data split used to fit and calibrate the score model, and let \(\mathcal{D}_{\mathrm{cal}}\subseteq \mathcal{D}_{\mathrm{real}}\) denote the subset used for calibration. In practice, \(\mathcal{D}_{\mathrm{cal}}\) may be a held-out calibration split or the same real-data split used for fitting when data is limited.

For realism, after fitting \((\mu_{\mathrm{real}},\Sigma_{\mathrm{real}})\), we compute
\[
\widehat E_{\mathrm{real}}(y_i)
:=
(y_i-\widehat\mu_{\mathrm{real}})^\top \widehat\Sigma_{\mathrm{real}}^{-1}(y_i-\widehat\mu_{\mathrm{real}})
\qquad\text{for } y_i\in\mathcal{D}_{\mathrm{cal}},
\]
and estimate
\[
\hat m_R:=\frac{1}{|\mathcal{D}_{\mathrm{cal}}|}\sum_{y_i\in\mathcal{D}_{\mathrm{cal}}}\widehat E_{\mathrm{real}}(y_i),
\qquad
\hat s_R:=\sqrt{\frac{1}{|\mathcal{D}_{\mathrm{cal}}|-1}\sum_{y_i\in\mathcal{D}_{\mathrm{cal}}}\big(\widehat E_{\mathrm{real}}(y_i)-\hat m_R\big)^2 }.
\]
The calibrated realism score is then
\[
\widehat R(y):=\frac{\widehat E_{\mathrm{real}}(y)-\hat m_R}{\hat s_R+\varepsilon_R},
\]
where \(\varepsilon_R>0\) is a small numerical stabilizer.

For faithfulness, fix an attribute \(k\) and target value \(t\in\mathcal{A}_k\). For each calibration sample \(y_i\) satisfying \(a_{i,k}=t\), compute the fitted margin
\[
\widehat M_k(y_i;t):=\widehat d_k(y_i;t)-\min_{v\neq t}\widehat d_k(y_i;v).
\]
Then estimate
\[
\hat m_{k,t}:=
\frac{1}{|\mathcal{D}_{k,t}^{\mathrm{cal}}|}
\sum_{y_i\in\mathcal{D}_{k,t}^{\mathrm{cal}}}\widehat M_k(y_i;t),
\]
and
\[
\hat s_{k,t}:=
\sqrt{
\frac{1}{|\mathcal{D}_{k,t}^{\mathrm{cal}}|-1}
\sum_{y_i\in\mathcal{D}_{k,t}^{\mathrm{cal}}}
\big(\widehat M_k(y_i;t)-\hat m_{k,t}\big)^2
},
\]
where
\[
\mathcal{D}_{k,t}^{\mathrm{cal}}:=\{y_i\in\mathcal{D}_{\mathrm{cal}}: a_{i,k}=t\}.
\]
The calibrated attribute score becomes
\[
\widehat F_k(y;t):=\frac{\widehat M_k(y;t)-\hat m_{k,t}}{\hat s_{k,t}+\varepsilon_F},
\]
with \(\varepsilon_F>0\) a small numerical stabilizer.

In practice, if a class-conditional calibration set is very small, one may use a minimum-count threshold before reporting calibrated scores, pool variance estimates across values of the same attribute, or clip the denominator away from zero. These are finite-sample stabilization choices and do not affect the population-level theory above.

\subsection{Shared-covariance estimation and shrinkage}

For each attribute \(k\), let
\[
\widehat\eta_{k,v}
=
\frac{1}{n_{k,v}}
\sum_{i:a_{i,k}=v} y_i
\]
denote the empirical prototype for value \(v\in\mathcal{A}_k\), where \(n_{k,v}\) is the number of real samples with \(a_{i,k}=v\).

The pooled within-value covariance is
\[
\widehat\Sigma_k^{\mathrm{within}}
:=
\frac{1}{N_k-|\mathcal{A}_k|}
\sum_{v\in\mathcal{A}_k}
\sum_{i:a_{i,k}=v}
(y_i-\widehat\eta_{k,v})(y_i-\widehat\eta_{k,v})^\top,
\]
where \(N_k=\sum_{v\in\mathcal{A}_k} n_{k,v}\).

To improve conditioning in high dimension, we use shrinkage toward an isotropic target:
\[
\widehat\Sigma_k^{\mathrm{shr}}
:=
(1-\alpha_k)\widehat\Sigma_k^{\mathrm{within}}
+
\alpha_k \tau_k I,
\qquad
\tau_k:=\frac{1}{d}\operatorname{tr}\big(\widehat\Sigma_k^{\mathrm{within}}\big),
\]
with shrinkage coefficient \(\alpha_k\in[0,1]\). The corresponding precision estimate is
\[
\widehat P_k:=\big(\widehat\Sigma_k^{\mathrm{shr}}+\varepsilon_P I\big)^{-1},
\]
where \(\varepsilon_P>0\) is a small numerical ridge added for inversion stability.

The resulting fitted quadratic distance is
\[
\widehat d_k(y;v):=(y-\widehat\eta_{k,v})^\top \widehat P_k (y-\widehat\eta_{k,v}).
\]

The same construction can be used for the global realism model by replacing the class-conditional pooled covariance with the empirical covariance of all real normalized features.

\subsection{Finite-sample considerations}

Several finite-sample issues are worth noting.

First, reference coverage is a population support condition. In finite data, the corresponding anchored slices may exist but contain few examples. This is one reason we use pooled prototypes in the practical method even when the theory is stated relative to a reference-anchored comparator.

Second, class imbalance affects both prototype estimation and calibration. For heavily imbalanced attributes, the empirical prototypes \(\widehat\eta_{k,v}\) and calibration moments \((\hat m_{k,t},\hat s_{k,t})\) may have different variances across values. This motivates reporting class counts and, if needed, applying minimum-support thresholds or pooled-variance stabilizers.

Finally, the shrinkage level \(\alpha_k\) trades bias for conditioning. Stronger shrinkage produces more stable inverses in high dimension and small sample regimes, while weaker shrinkage preserves finer geometry when sufficient data are available.

\section{Broader impacts}
  \label{app:broader-impacts}

  \paragraph{Positive impacts.}
  The proposed trust score is intended to make conditional generative models
  more deployable in scientific settings where target real samples are
  unavailable, particularly biological imaging where wet-lab validation of
  \emph{in silico} predictions is expensive. By providing a per-sample,
  per-condition reliability estimate, the score can prevent unreliable synthetic
  generations from being treated as evidence in downstream analyses, reducing
  the risk of misleading scientific conclusions and saving experimental
  resources by prioritizing high-trust predictions for validation.

\paragraph{Potential negative impacts.}
Any post-hoc filtering rule can be misused to cherry-pick generations that look plausible without being faithful to the requested condition --- our score mitigates but does not eliminate this risk, and we encourage users to report acceptance rates and per-condition trust calibration alongside any selected sample sets.

\paragraph{Dual-use considerations.} The score itself is a diagnostic and does not improve the generative capabilities of the underlying model; it does not lower the barrier to producing higher-quality synthetic faces or biological images relative to running the underlying generator alone.

\section{Implementation details}
  \label{app:impl-details}

  \paragraph{Architecture.}
  SiT-B/2 \citep{ma2024sit} backbone (12 layers, 768 hidden dim, 12 heads),
  patch size 2 over a $32{\times}32$ VAE latent (256 patches), stable-diffusion
  SD-VAE-MSE encoder/decoder \citep{rombach2022latentdiff} (frozen). Conditioning
  is per-attribute embedding summed with the timestep embedding. The REPA
  projector is a 3-layer MLP $768\to2048\to2048\to d_h$ inserted at encoder
  depth $\ell=8$, where $d_h\in\{1024,1152,384\}$ for DINOv3 / SigLIP /
  OpenPhenom respectively.

  \paragraph{Training.}
  AdamW with $\mathrm{lr}=10^{-4}$, weight decay $0$, $1000$ linear warmup
  steps, gradient clip $1.0$, $400{,}000$ optimizer steps, global batch size
  $128$ ($16$ per GPU $\times$ 8 GPUs, DDP, \texttt{bf16} mixed precision),
  EMA on the generator with decay $0.9999$. Class-conditional dropout
  $p=0.1$ for classifier-free-guidance training; for marginal/held-out runs
  we additionally use the compositional regularizer with weight $1.0$ applied
  to $25\%$ of batches at $t>0.7$.

  \paragraph{Data and preprocessing.}
  \celeba: $256{\times}256$ centre-cropped RGB, random horizontal flip,
  $[-1,1]$ normalization; conditioning on the 4 binary attributes
  \textsc{Male / Smiling / Blond\_Hair / Eyeglasses}. \rxrx: $512{\times}512$
  6-channel raw fluorescence resized to $256{\times}256$ via numpy bilinear,
  no flips, no per-channel normalization (raw uint16 scaled to $[0,1]$);
  conditioning on (cell\_type\_id, sirna\_id).

  \paragraph{Sampling.}
  $250$-step DDIM-style integration of the SiT flow with $t_{\mathrm{cutoff}}=0.04$;
  classifier-free-guidance scale $1.0$. We draw $1000$
  samples per condition for \celeba (16 conditions) and $100$ per condition for
  \rxrx (50 conditions in the canonical subset).

  \paragraph{Trust-score calibration.}
  Encoder features are taken as the mean over patch tokens (DINOv3 / SigLIP)
  or as the model's pooled output (OpenPhenom). Covariance estimates use
  Ledoit--Wolf shrinkage \citep{ledoit2004shrinkage} with $\varepsilon_P=10^{-5}$ ridge added before inversion.
  Calibration constants $(\hat m_R,\hat s_R,\hat m_{k,t},\hat s_{k,t})$ are
  estimated on the same real training-support split used for fitting; full
  formulas are in \Cref{app:estimation-details}.
  
\section{During-generation \transl and compute details}
\label{app:compute-details}

The \transl is trained to preserve the aspects of the \fext representation used by the Mahalanobis trust score. A standard cosine-alignment loss is sufficient to align dominant directions, but it need not preserve the residual geometry that controls whitened distances and prototype comparisons. This can distort the relative scale of directions that are weak in cosine similarity but important after covariance normalization. We therefore train \(g_{\phi,\ell}\) with a whitened-geometry objective: a tempered feature-matching term in feature-whitened coordinates, together with mean and covariance matching penalties in normalized feature space. The whitening transform is computed from the real training features, using the same covariance geometry as the trust score. This makes the mapped features compatible with the post-generation scoring rule, so that any difference between post-generation and during-generation evaluation reflects feature acquisition rather than a change in metric.

Specifically, we train the \transl from an internal model feature \(z_{\ell,\tau}(x)\) to the \fop feature space. Let
\[
p_\phi(x)=g_{\phi,\ell}(z_{\ell,\tau}(x)),
\qquad
\bar p_\phi(x)=\operatorname{norm}(p_\phi(x)),
\qquad
\bar y(x)=\operatorname{norm}(\fop(x)).
\]
Let \(\mu_t\) and \(\Sigma_t\) denote the empirical mean and covariance of normalized pretrained features on real training data. We use
\[
\mathcal L_{\mathrm{whitgeom}}
=
\lambda_{\mathrm{whit}}
\mathbb E_x
\left[
\left\|
W\left(\bar p_\phi(x)-\bar y(x)\right)
\right\|_2^2
\right]
+
\lambda_{\mathrm{geom}}
\left(
\frac{\|\mu_p-\mu_t\|_2^2}{\|\mu_t\|_2^2+\epsilon}
+
\frac{\|\Sigma_p-\Sigma_t\|_F^2}{\|\Sigma_t\|_F^2+\epsilon}
\right),
\]
where \(\mu_p\) and \(\Sigma_p\) are the batch mean and covariance of mapped normalized features. The matrix \(W\) is a tempered whitening transform computed from the feature covariance, with exponent \(\gamma=0.75\) and condition-number cap \(\kappa=1000\). In all experiments we use
\[
\lambda_{\mathrm{whit}}=1.0,
\qquad
\lambda_{\mathrm{geom}}=0.1.
\]

\subsection{Compute save from early abstention}

Let \(C_{\mathrm{den}}(\ell,\tau)\) denote the cost of running the generator up to layer-step pair \((\ell,\tau)\), \(C_{\mathrm{dec}}\) the decoder cost, \(C_{\mathrm{enc}}\) the \fop cost, and \(C_{g_\phi}\) the \transl cost. Post-generation scoring costs
\[
C_{\mathrm{post}}
=
C_{\mathrm{den}}(\mathrm{full})
+
C_{\mathrm{dec}}
+
C_{\mathrm{enc}}
+
C_{\mathrm{score}},
\]
whereas during-generation scoring costs
\[
C_{\ell,\tau}^{\mathrm{map}}
=
C_{\mathrm{den}}(\ell,\tau)
+
C_{g_\phi}
+
C_{\mathrm{score}}.
\]
Thus early abstention avoids the remaining denoising computation, the VAE decoder, and the \fext.

\subsection{Hardware, runtime, and total compute}
  \label{app:compute-resources}

  All experiments were run on a single internal cluster node with
  8$\times$NVIDIA A100 80GB GPUs (DDP, NCCL backend, mixed-precision
  \texttt{bf16}). Generator training uses 400{,}000 optimizer steps with
  \texttt{lr=1e-4}, \texttt{warmup\_steps=1000}, \texttt{gradient\_clip\_val=1.0},
  and global batch size 128 (16 per GPU $\times$ 8 GPUs).

  \paragraph{Per-experiment cost.} Generations are performed with the 250-step DDIM-style
  trajectory: $1000$ samples per condition for CelebA (16 conditions $\Rightarrow$
  $16\text{k}$ samples) and $100$ per condition for RxRx1 (50 conditions
  $\Rightarrow$ $5\text{k}$ samples). Per-step generator FLOPs are
  $\approx46$\,GFLOPs (SiT-B/2, 256 patches), so a full 250-step trajectory is
  $\approx 11.5$\,TFLOPs per sample; the REPA projector adds
  $\approx 4$\,GFLOPs (\textless 0.04\% overhead).

\paragraph{Trust-evaluation cost.} Once features are cached, all reported trust evaluations run on a single GPU in under an hour per (model, scoring space, regime).

\section{OpenPhenom feature diagnostic}
\label{app:openphenom-diagnostic}

While OpenPhenom-aligned REPA models achieve the best generation quality on RxRx1 as measured by post-hoc DINOv3 KID (\Cref{tab:condition-correlation}), trust scoring in the OpenPhenom-aligned feature space fails ($\rho \approx 0.05$). We conducted a systematic diagnostic to understand this failure.

\paragraph{Root cause: collapsed within-class covariance.}
OpenPhenom features have an effective within-class dimensionality of ${\approx}4$ in 2304 dimensions, compared to ${\approx}16$ for DINOv3. The top principal component captures $61.5\%$ of within-class variance, and this dominant direction is orthogonal to both cell type and perturbation identity. The precision matrix therefore has condition number ${\approx}21\text{M}$, amplifying ${\approx}2300$ near-zero variance directions.

\paragraph{Mahalanobis scoring destroys signal.}
The precision-weighted signal-to-noise ratio (Mahalanobis SNR) is $0.11$ for OpenPhenom vs.\ $0.24$ for DINOv3. The Mahalanobis/Euclidean SNR ratio is $0.73$ for OpenPhenom (signal-destructive) vs.\ $1.55$ for DINOv3 (signal-enhancing). This places OpenPhenom features in the weak-robustness regime of \Cref{prop:interaction-robustness}: the noise term $2\max_{u \neq t}|(\eta_{k,u} - \eta_{k,t})^\top P_k \lambda|$ exceeds the signal term $\min_{u \neq t}\|\eta_{k,u} - \eta_{k,t}\|_{P_k}^2$ for 77.8\% of samples on the siRNA attribute (\Cref{tab:prop4-test}).

\begin{table}[t]
\centering
\caption{\Cref{prop:interaction-robustness} robustness margin test across encoders. Violation rate = fraction of samples where noise exceeds signal in the Mahalanobis margin. N/S = noise-to-signal ratio.}
\label{tab:prop4-test}
\small
\begin{tabular}{l ccc ccc}
\rowcolor{gray!20}\toprule
 & \multicolumn{3}{c}{cell\_type attribute} & \multicolumn{3}{c}{siRNA attribute} \\
\cmidrule(lr){2-4} \cmidrule(lr){5-7}
Encoder & Violation & N/S & Slack & Violation & N/S & Slack \\
\midrule
DINOv3 & 0.0\% & 0.10 & 56.4 & 69.5\% & 1.21 & $-$4.7 \\
SigLIP & 0.0\% & 0.07 & 63.1 & 70.5\% & 1.26 & $-$5.3 \\
OpenPhenom & 0.7\% & 0.18 & 32.5 & 77.8\% & 1.48 & $-$6.3 \\
\bottomrule
\end{tabular}
\end{table}

\paragraph{Preprocessing drives the geometry.}
An isolation experiment applying OpenPhenom-style preprocessing (InstanceNorm + per-channel encoding) to DINOv3 eliminates the cross-space anti-correlation ($\rho$: $-0.54 \to +0.12$) but simultaneously destroys DINOv3's same-space scoring power ($\rho$: $0.93 \to 0.43$); see \Cref{tab:preprocessing-isolation}. This confirms that the issue is a fundamental consequence of OpenPhenom's channel-agnostic architecture: InstanceNorm removes per-image intensity variation, and per-channel processing prevents cross-channel interaction, concentrating variance into non-semantic directions.

\begin{table}[t]
\centering
\caption{Effect of OpenPhenom-style preprocessing on DINOv3 features. Cross-space = OP trust vs.\ DINOv3 KID; same-space = DINOv3 trust vs.\ DINOv3 KID. IN = InstanceNorm.}
\label{tab:preprocessing-isolation}
\small
\begin{tabular}{l cc}
\rowcolor{gray!20}\toprule
DINOv3 preprocessing variant & $\rho$(cross-space) & $\rho$(same-space) \\
\midrule
Standard (to\_rgb) & $-$0.54 & 0.93 \\
+ InstanceNorm & $-$0.60 & 0.80 \\
Dual 3ch (no IN) & $-$0.65 & 0.77 \\
Dual 3ch + IN & +0.16 & 0.43 \\
Per-channel (no IN) & $-$0.40 & 0.43 \\
Per-channel + IN & +0.12 & 0.56 \\
\bottomrule
\end{tabular}
\end{table}

\paragraph{Contrastive models are robust.}
A control experiment with post-hoc SigLIP features on the same RxRx1 images achieves $\rho = 0.80$ cross-space (SigLIP trust vs.\ DINOv3 KID), despite SigLIP having even higher pairwise cosine similarity ($0.979$ vs.\ $0.95$ for OpenPhenom). The problem is isolated to OpenPhenom's MAE feature geometry, not to high feature similarity per se (\Cref{tab:encoder-comparison}).

\begin{table}[t]
\centering
\caption{Cross-encoder comparison on RxRx1 (50 conditions, same generated images from REPA-OpenPhenom model). Same-space = trust and KID in the same encoder; cross-space = encoder trust vs.\ DINOv3 KID.}
\label{tab:encoder-comparison}
\small
\begin{tabular}{l cc cc}
\rowcolor{gray!20}\toprule
 & \multicolumn{2}{c}{Same-space} & \multicolumn{2}{c}{Cross-space (vs.\ DINOv3 KID)} \\
\cmidrule(lr){2-3} \cmidrule(lr){4-5}
Encoder & $\rho$(trust) & $\rho$(real.) & $\rho$(trust) & $\rho$(real.) \\
\midrule
DINOv3 & 0.93 & 0.94 & --- & --- \\
SigLIP & 0.79 & 0.65 & 0.80 & 0.84 \\
OpenPhenom & 0.45 & 0.40 & $-$0.54 & $-$0.48 \\
\bottomrule
\end{tabular}
\end{table}

\paragraph{Feature geometry comparison.}
\Cref{tab:feature-geometry} summarizes the within-class covariance geometry that underlies the scoring failure.

\begin{table}[t]
\centering
\caption{Within-class covariance geometry across encoders (RxRx1 real features).}
\label{tab:feature-geometry}
\small
\begin{tabular}{l ccc}
\rowcolor{gray!20}\toprule
Metric & DINOv3 & SigLIP & OpenPhenom \\
\midrule
Feature dimension & 1024 & 1152 & 2304 \\
Effective dim (within-class) & 16.0 & 12.2 & 4.0 \\
Top-1 PC fraction & 0.19 & 0.23 & 0.62 \\
Precision condition number & 5.1M & 1.5M & 21.1M \\
Euclidean SNR & 0.15 & 0.11 & 0.15 \\
Mahalanobis SNR & 0.24 & 0.28 & 0.11 \\
Mahal/Euclid ratio & 1.55 & 2.52 & 0.73 \\
\bottomrule
\end{tabular}
\end{table}

\paragraph{Practical implication.}
OpenPhenom features are discriminative for classification ($\eta^2 \approx 0.41$ for cell type, comparable to DINOv3's $0.36$) but violate the geometric assumptions required for Mahalanobis-based distributional scoring. This highlights a key distinction: classification requires only hyperplane separability in a few discriminative dimensions, while trust scoring requires meaningful geometry across the full feature space. For trust evaluation on RxRx1, we recommend using post-hoc contrastive features (DINOv3 or SigLIP) rather than MAE-based domain-specific encoders.

\section{RxRx1 50-condition evaluation subset}
\label{app:rxrx1-subset}

All RxRx1 results in the main text are reported on a balanced 50-condition subset (25 seen + 25 unseen, covering all 4 cell types) rather than the full 4552-condition space. This appendix documents why the subset is constructed the way it is, and why a naive ``top-$n$ by sample count, stratified by cell type'' construction fails.

\paragraph{Constraints.}
A usable subset must satisfy three requirements simultaneously. (i)~\emph{Sample sufficiency}: each retained condition must have enough real samples for a per-condition KID bootstrap to be stable. (ii)~\emph{Cell-type coverage}: all four cell types must appear in the subset, so that trust-vs-$\Delta$KID correlations are not driven by a single biological context. (iii)~\emph{Discriminability}: each condition must be distinguishable from the rest under a class-balanced linear probe on real feature space, otherwise trust scoring has nothing to rank, and apparent correlations are dominated by class-imbalance artifacts rather than genuine per-condition structure. The held-out training regime (\Cref{sec:experiments}) adds a fourth constraint: the test set should include enough unseen conditionings for it to be stress-testing, the 100 most-frequent cell~type\,$\times$\,perturbation pairs removed from training.

\paragraph{Failure mode of the naive construction.}
A first-pass subset was built by selecting the top-$n$ conditions by real-sample count, stratified across cell types (7/6/6/6 seen across ct=0/1/2/3 and 9/8/8 unseen across ct=0/1/2). A cross-validated class-balanced linear probe on real DINOv3 features, run on this subset, exposed three problems.

\begin{enumerate}[nosep]
    \item \emph{Filler rows.} 13 of the 25 seen conditions (all ct=0 and ct=2 non-controls) had only $n{=}14$ real samples each. This is below any stable per-condition KID bootstrap. Under a class-balanced probe on the 50-class subproblem, these rows collapse to top-1 ${\approx}0$---the probe cannot distinguish them from each other or from the rest of the catalog. They contribute no usable ranking signal.
    \item \emph{No unseen functional diversity.} The unseen part of the test set contained only control wells, even though several non-control held-out perturbations were available in held-out set of the diffusion training. The support-shift test therefore covered a single class of perturbations rather than the diversity the held-out set was designed to provide.
    \item \emph{Imbalance-probe confound.} A class-imbalanced probe resulted in a class-frequency artifact: controls dominate the per-class sample counts. Under a class-balanced probe restricted to the 50-class subproblem, the controls collapse and the genuinely discriminable rows are the non-control siRNAs together with a smaller subset of controls. Any subset that does not prioritize these rows is measuring imbalance, not trust
\end{enumerate}

\paragraph{Corrected construction.}
The subset used throughout the paper is built to satisfy all four constraints by explicitly targeting discriminable rows under a class-balanced probe, rather than largest-$n$ rows.

\begin{itemize}[nosep]
    \item \emph{Seen arm (25 conditions, $n \geq 23$, median $n{=}32$).} 20 ct=1 non-control siRNAs selected by real-sample count (all at $n{=}32$), plus 5 ct=3 controls at $n \approx 22$--$24$. The ct=1 non-controls carry the per-condition signal identified by the class-balanced probe; the ct=3 rows satisfy the cell-type-coverage constraint (see below).
    \item \emph{Unseen arm (25 conditions, $n \geq 27$, median $n \approx 52$).} 8 ct=0 controls (largest by $n$), 7 ct=1 non-control held-out siRNAs (349, 350, 351, 352, 363, 364, 401), 2 ct=1 controls (largest by $n$), and 8 ct=2 controls (largest by $n$). This is the first configuration in which the unseen arm contains non-control perturbations, giving the support-shift evaluation functional rather than purely positional variety.
\end{itemize}

\paragraph{Cell-type coverage notes.}
Two cell-type asymmetries are worth making explicit.
\begin{itemize}[nosep]
    \item \emph{ct=0 and ct=2 appear only in the unseen arm.} Their non-control conditions cap at $n{=}14$, which fails the sample-sufficiency constraint, and all of their high-$n$ controls are listed in \texttt{RXRX1\_HELDOUT\_PAIRS}, so they cannot populate the seen arm. Subset-level cell-type coverage (each cell type appears somewhere in the 50 conditions) is therefore retained, even though arm-level coverage is not---the alternative would require reinstating the 13 filler rows.
    \item \emph{ct=3 is always retained.} ct=3 has no held-out pairs at all (so it cannot populate the unseen arm) and its non-control wells max out at $n{=}6$ (unusable). The only ct=3 data with sufficient sample counts are the 31 control wells (\texttt{sirna\_id} 1108--1138, max $n{=}24$). The seen arm therefore includes the 5 largest ct=3 controls unconditionally, as the only way to honor the all-cell-types constraint.
\end{itemize}

\paragraph{Relationship to the naive subset.}
29 of the 50 final conditions are also present in the naive construction (18 in the unseen arm, 11 in the seen arm), and no condition migrates between arms. The corrected subset is therefore a targeted sharpening of the original---dropping filler rows and adding discriminable rows---rather than a rebuild, so results under the two constructions remain directly comparable on their shared conditions.

\section{CellProfiler validation details on RxRx1}
\label{app:cp-validation}

This appendix details the CellProfiler (CP) feature pipeline used for the morphology-space validation in \Cref{sec:cp-validation} and reports the per-feature correlation breakdown and uninformative-feature sanity check.

\paragraph{Decile-binning headline numbers.}
\Cref{tab:cp-decile-downstream} summarizes the curves of \Cref{fig:cp-decile-downstream} as the bin-0 vs.\ bin-9 micro-accuracy (the headline contrast referenced in the main text), per (model, target).

\begin{table}[h]
\centering
\caption{CP-space decile downstream classification on RxRx1 (\Cref{fig:cp-decile-downstream} headline numbers; 50-condition subset, marginal models, SigLIP scoring). Micro-accuracy of bin-0 vs.\ bin-9 classifiers trained on per-decile gen samples and tested on the canonical-50 real pool. ``Trust spread'' = bin~0 $-$ bin~9; positive means higher-trust deciles produce more useful training data. Random-baseline micro-accuracy is flat (std $\leq 0.005$) and matches the bin-0/bin-9 mean.}
\label{tab:cp-decile-downstream}
\small
\begin{tabular}{ll cccc}
\rowcolor{gray!20}\toprule
Model & Target & bin 0 & bin 9 & Trust spread & Random mean \\
\midrule
Vanilla       & celltype & 0.900 & 0.841 & $+$0.058 & 0.893 \\
Vanilla       & combo    & 0.127 & 0.109 & $+$0.018 & 0.130 \\
REPA (DINOv3) & celltype & 0.905 & 0.814 & $+$0.091 & 0.887 \\
REPA (DINOv3) & combo    & 0.148 & 0.107 & $+$0.041 & 0.132 \\
REPA (SigLIP) & celltype & 0.897 & 0.803 & $+$0.094 & 0.887 \\
REPA (SigLIP) & combo    & 0.146 & 0.117 & $+$0.029 & 0.131 \\
\bottomrule
\end{tabular}
\end{table}

\begin{figure}[h]
\centering
\begin{subfigure}[t]{\textwidth}
    \centering
    \includegraphics[width=\textwidth]{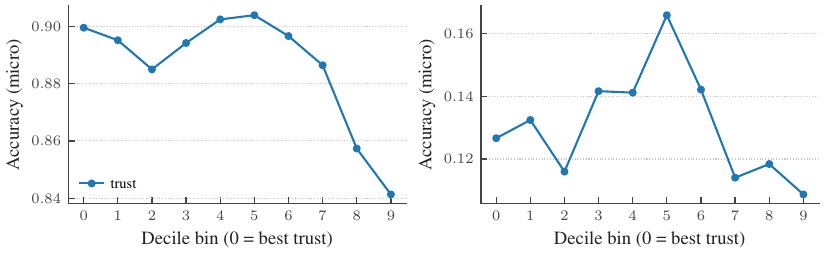}
    \caption{Vanilla marginal.}
    \label{fig:cp-decile-vanilla}
\end{subfigure}
\\[0.4em]
\begin{subfigure}[t]{\textwidth}
    \centering
    \includegraphics[width=\textwidth]{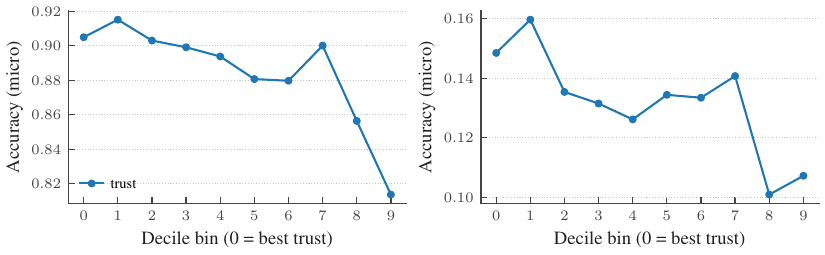}
    \caption{REPA-DINOv3 marginal.}
    \label{fig:cp-decile-repa}
\end{subfigure}
\\[0.4em]
\begin{subfigure}[t]{\textwidth}
    \centering
    \includegraphics[width=\textwidth]{figures/rxrx1_cp_decile_repa_siglip_marginal.pdf}
    \caption{REPA-SigLIP marginal.}
    \label{fig:cp-decile-repa-siglip}
\end{subfigure}
\caption{\textbf{Full CP-space decile downstream classification on RxRx1 (kept-621 features, SigLIP trust scoring).} Each row is one marginal model; \emph{left}: 4-way celltype accuracy by decile, \emph{right}: 50-way combo accuracy. Solid blue: trust-ranked decile. The REPA-SigLIP row is highlighted in the main text.}
\label{fig:cp-decile-downstream}
\end{figure}

\paragraph{Per-sample centroid distance (full table).}
\Cref{tab:cp-centroid-full} reproduces \Cref{tab:cp-centroid} from the main text with the trust-accepted sample count $n_{\mathrm{trust}}$ and the bootstrap 95\% CI on $\Delta d$ (2000 reps, stratified). Negative $\Delta d$ with a CI bounded above by zero is read as ``trust-selected samples are morphologically closer to real''; the vanilla-unseen row is the only cell whose CI grazes zero.

\begin{table}[h]
\centering
\caption{Full version of \Cref{tab:cp-centroid}: per-sample centroid distance on RxRx1 (kept-621 CP features) with trust-accepted counts and bootstrap CIs. Random pool is condition-matched (combo-stratified) and matched in size to $n_{\mathrm{trust}}$ per arm.}
\label{tab:cp-centroid-full}
\small
\begin{tabular}{ll r r ccc c}
\rowcolor{gray!20}\toprule
Model & Arm & $n_{\mathrm{trust}}$ & Accept\% & $\bar d_{\mathrm{trust}}$ & $\bar d_{\mathrm{baseline}}$ & $\Delta d$ & 95\% CI \\
\midrule
Vanilla       & seen   & 1657 & 66.3 & 18.63 & 18.93 & $-$0.300 & $[-0.570,\, -0.020]$ \\
Vanilla       & unseen &  547 & 21.9 & 20.92 & 21.52 & $-$0.601 & $[-1.211,\, +0.028]$ \\
REPA (DINOv3) & seen   & 1530 & 61.2 & 18.31 & 18.90 & $-$0.586 & $[-0.854,\, -0.314]$ \\
REPA (DINOv3) & unseen &  524 & 21.0 & 20.19 & 21.21 & $-$1.018 & $[-1.663,\, -0.394]$ \\
REPA (SigLIP) & seen   & 1640 & 65.7 & 18.38 & 18.88 & $-$0.496 & $[-0.750,\, -0.241]$ \\
REPA (SigLIP) & unseen &  588 & 23.5 & 22.13 & 22.97 & $-$0.844 & $[-1.489,\, -0.167]$ \\
\bottomrule
\end{tabular}
\end{table}

\paragraph{Feature pipeline.}
Starting from the raw \texttt{Image.csv} produced by CellProfiler over the canonical-50 real subset, we apply the following deterministic pipeline. (i)~Drop metadata columns (regex matches against \texttt{metadata}, \texttt{filename}, \texttt{pathname}, \texttt{imagenumber}, \texttt{objectnumber}, \texttt{executiontime}, \texttt{moduleerror}, \texttt{series}, \texttt{frame}, \texttt{group\_*}, \texttt{url}), leaving 2467 numeric features. (ii)~Drop NaN rows and restrict to the canonical 50 cell$\times$siRNA pairs (2061 real samples). (iii)~Apply a variance threshold ($\mathrm{var} < 10^{-5}$) and an outlier filter (drop any column with a real-data $|z| > 5$), giving the \emph{kept-621} set used throughout the centroid and per-feature correlation analyses. (iv)~Real-fit \texttt{StandardScaler} is reused across all readouts.

For the uninformative-feature sanity check we use a less aggressive variant — \emph{unfiltered} — which keeps all 2415 columns that survive metadata removal and zero-variance pruning, so genuinely noise-only features can surface.

\paragraph{Per-feature correlation construction.}
For each (model, arm, feature $f$),
\[
r_{\mathrm{trust}}(f) = \mathrm{Pearson}\!\big(\mu_{\mathrm{trust}}(f),\,\mu_{\mathrm{real}}(f)\big),
\qquad
r_{\mathrm{rand}}(f) = \mathrm{Pearson}\!\big(\mu_{\mathrm{rand}}(f),\,\mu_{\mathrm{real}}(f)\big),
\]
where $\mu_\bullet(f)$ is the per-combo mean of $f$ over the corresponding gen subset (trust-accepted, or combo-stratified random averaged over 5 seeds), and the Pearson correlation is taken across the arm's combos. Higher $r$ means the feature's per-combo profile in gen tracks real well; $\Delta(f) := r_{\mathrm{trust}}(f) - r_{\mathrm{rand}}(f) > 0$ means trust selection sharpens cross-combo correspondence on that feature.

\paragraph{Top-correlated features under trust (kept-621).}
Across all six (model, arm) cells the same families recur at the top of $r_{\mathrm{trust}}$ ($r \geq 0.90$): per-cell and per-cytoplasm texture autocorrelation at scales 3/5/10 (\texttt{Cells\_Texture\_Correlation\_OrigGray\_*}), image-level granularity (\texttt{Granularity\_2/4\_OrigGray}), total cell-segmentation area (\texttt{AreaOccupied\_AreaOccupied\_IdentifySecondaryObjects}), and cytoplasm shape regularity variance (\texttt{StDev\_Cytoplasm\_AreaShape\_FormFactor}). \Cref{tab:cp-top-features} gives a representative example for the REPA-DINOv3 unseen arm.

\begin{table}[t]
\centering
\caption{Top-5 features by $r_{\mathrm{trust}}$ on REPA-DINOv3, RxRx1 unseen arm (kept-621). $\Delta = r_{\mathrm{trust}} - r_{\mathrm{rand}}$.}
\label{tab:cp-top-features}
\small
\begin{tabular}{l ccc}
\rowcolor{gray!20}\toprule
Feature & $r_{\mathrm{trust}}$ & $r_{\mathrm{rand}}$ & $\Delta$ \\
\midrule
\texttt{Median\_Cells\_Texture\_Correlation\_OrigGray\_5\_02\_256}     & $+$0.937 & $+$0.926 & $+$0.011 \\
\texttt{Median\_Cells\_Texture\_Correlation\_OrigGray\_10\_02\_256}    & $+$0.935 & $+$0.836 & $+$0.099 \\
\texttt{Median\_Cytoplasm\_Texture\_Correlation\_OrigGray\_5\_02\_256} & $+$0.934 & $+$0.863 & $+$0.071 \\
\texttt{Mean\_Cells\_Texture\_Correlation\_OrigGray\_5\_03\_256}       & $+$0.934 & $+$0.893 & $+$0.041 \\
\texttt{Mean\_Cells\_Texture\_Correlation\_OrigGray\_5\_01\_256}       & $+$0.932 & $+$0.921 & $+$0.011 \\
\bottomrule
\end{tabular}
\end{table}

\paragraph{Uninformative-feature sanity check (unfiltered).}
Repeating the analysis on the unfiltered 2415-feature set surfaces features for which even random gen has no per-combo correspondence with real. Sorting by $|r_{\mathrm{rand}}|$ ascending isolates these. \Cref{tab:cp-uninformative} shows the bot-10 for the vanilla unseen arm; the same pattern appears across all six (model, arm) cells.

\begin{table}[t]
\centering
\caption{Bot-10 features by $|r_{\mathrm{rand}}|$ on Vanilla, RxRx1 unseen arm (unfiltered 2415-feature set). Random gen has no per-combo correspondence with real on these features ($|r_{\mathrm{rand}}| \approx 0$), consistent with their interpretation as uninformative under any selection.}
\label{tab:cp-uninformative}
\small
\begin{adjustbox}{width=\linewidth}
\begin{tabular}{l cc l}
\rowcolor{gray!20}\toprule
Feature & $r_{\mathrm{trust}}$ & $r_{\mathrm{rand}}$ & Family \\
\midrule
\texttt{Mean\_IdentifyPrimaryObjects\_Children\_..\_Count}             & $-$0.000 & $-$0.000 & tautological count \\
\texttt{StDev\_IdentifyPrimaryObjects\_Children\_..\_Count}            & $-$0.000 & $-$0.000 & tautological count \\
\texttt{StDev\_Nuclei\_Texture\_Variance\_OrigGray\_5\_01\_256}        & $-$0.086 & $+$0.001 & texture variability \\
\texttt{Median\_Nuclei\_RadialDistribution\_RadialCV\_..\_4of4}        & $-$0.047 & $+$0.001 & outermost-shell radial CV \\
\texttt{Median\_Cells\_Granularity\_15\_OrigGray}                      & $-$0.262 & $-$0.001 & very-coarse granularity \\
\texttt{StDev\_Nuclei\_AreaShape\_Zernike\_9\_7}                       & $-$0.269 & $-$0.002 & high-order Zernike \\
\texttt{Median\_Cytoplasm\_Texture\_InfoMeas1\_OrigGray\_10\_01\_256}  & $+$0.052 & $-$0.004 & texture moment \\
\texttt{Median\_Cytoplasm\_AreaShape\_FormFactor}                       & $+$0.127 & $-$0.004 & shape regularity \\
\texttt{Mean\_Cells\_AreaShape\_Zernike\_7\_1}                         & $+$0.282 & $-$0.004 & high-order Zernike \\
\texttt{StDev\_Cells\_Neighbors\_AngleBetweenNeighbors\_5}             & $+$0.080 & $+$0.005 & neighbour angle \\
\bottomrule
\end{tabular}
\end{adjustbox}

\end{table}

The recurring uninformative families across all six (model, arm) cells are: tautological segmentation counts (every primary object has exactly one secondary child after segmentation, so the column is constant), very-coarse-scale granularity (dominated by well-level lighting), high-order Zernike moments (sample-noise-dominated rotational shape), outermost-ring radial CVs, and absolute object orientation (intrinsically arbitrary). None of these should correlate with biology, and $r_{\mathrm{rand}} \approx 0$ confirms they don't.

A side observation: a handful of features that are uninformative under random become \emph{strongly} anti-correlated under trust ($r_{\mathrm{trust}} \ll 0$ while $|r_{\mathrm{rand}}| \approx 0$). The clearest example is \texttt{Mean\_Cytoplasm\_Location\_MaxIntensity\_Y\_OrigGray} on REPA-SigLIP unseen ($r_{\mathrm{rand}} = +0.005$, $r_{\mathrm{trust}} = -0.475$): image-level position features that random gen treats as noise become systematically biased by trust selection. This is the same pose/position bias visible in the kept-621 anti-correlation list.

\section{Graded compositional shift on CelebA}
\label{app:graded-shift}

A natural follow-up question to the post-generation ordering result is whether unseen conditionings should be treated as uniformly unreliable. CelebA held-out gives this question a sharp form: each queried condition has a natural notion of shift severity given by its minimum Hamming distance to the observed training support, and we can ask whether the per-condition mean trust score follows the same gradient.

\Cref{fig:celeba-scatter} shows the answer. The five seen (single-attribute) conditions cluster at low $\Delta$KID and low trust; unseen conditions spread outward as the Hamming distance from support grows. Crucially, this is not a binary seen-vs.-unseen effect: some unseen conditions at small Hamming distance achieve quality comparable to seen ones, and the score correctly assigns them correspondingly favorable trust. Conversely, larger-distance compositions are both harder to generate and assigned worse trust.

This behavior matters for the intended use case. Under compositional generalization, the relevant question is not simply whether a request is outside training support, but how severe that shift is and whether the resulting sample should still be trusted. The score captures this graded structure directly, rather than collapsing all unseen combinations into a single failure category.

\begin{figure}[h]
\centering
\includegraphics[width=0.65\textwidth]{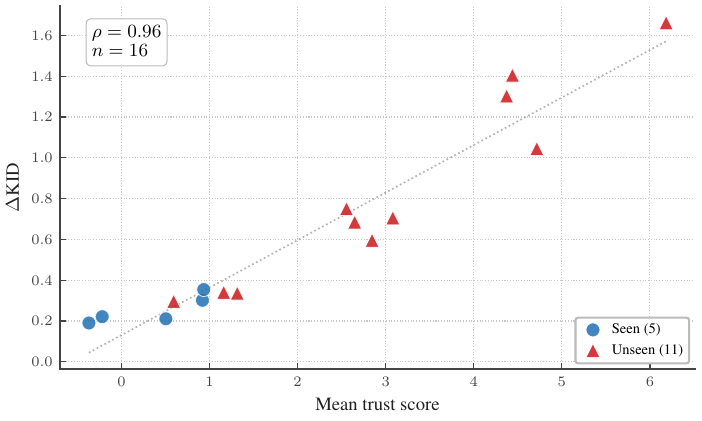}
\caption{\textbf{CelebA stress-test (REPA DINOv3, DINOv3 scoring): per-condition trust vs.\ $\Delta$KID.} Each point is one of 16 conditions. The 5 seen (single-attribute) conditions cluster at low $\Delta$KID; unseen conditions with more active attributes drift further from training support and exhibit higher $\Delta$KID. Trust scores track this degradation ($\rho = 0.96$).}
\label{fig:celeba-scatter}
\end{figure}

\Cref{tab:fpr95-selection-ablation} compares our calibrated Mahalanobis trust score (\Cref{sec:method}) against three alternative scorers on the same DINOv3 mean-patch feature space, under the identical P95-real-threshold selection protocol used for \Cref{tab:fpr95-selection}. \emph{Linear probe} fits one logistic regression per attribute on real DINOv3 features and scores a generated sample by the summed free energy $-\sum_a \mathrm{logsumexp}(z_a(x))$. \emph{kNN (per-attr.)} scores by the sum, over attributes, of cosine distance from the sample to its 5-th nearest neighbour inside the target-class real subset. \emph{CLIP} (CelebA-only) scores by $-\cos(\text{image}, \text{prompt})$ over a fixed template across the 16 attribute combos.

The picture is consistent with the takeaway in \Cref{sec:fpr95}. All three alternative scorers produce near-degenerate P95-real thresholds — the generated trust scores sit almost entirely below the 95\,th-percentile real threshold, yielding $\approx\!100\%$ acceptance — so they are unusable as abstention rules even when their ranking happens to be informative. kNN recovers a Mahalanobis-comparable condition-level correlation but cannot translate it into a calibrated filter; the linear probe and CLIP alignment collapse on both axes. Our Mahalanobis score is the only construction here that yields both non-degenerate acceptance and substantial KID improvement, which is what the deployment-time abstention story requires. The Mahalanobis rows in \Cref{tab:fpr95-selection-ablation} are reproduced from \Cref{tab:fpr95-selection,tab:fpr95-selection-full}; held-out rows match the main text and full-support rows are kept here as sanity checks.

\clearpage

\begin{table}[H]
\centering
\caption{Scorer ablation: P95-real-threshold sample selection across four alternative trust scorers on the same DINOv3 meanpatch feature space. Threshold, KID computation ($k{=}500$ shuffled subsampling), and random baseline are identical to \Cref{tab:fpr95-selection}. \emph{Mahalanobis} (ours) is the global-energy + per-attribute margin scorer in \Cref{sec:method}. $\rho$(T,$\Delta$KID) is the Spearman correlation between per-condition mean trust score and per-condition $\Delta$KID ($n{=}16$ for CelebA, $n{=}50$ for RxRx1). All three ablation scorers produce near-degenerate P95-real thresholds (gen scores sit almost entirely below the 95\,th-percentile real threshold), yielding $\approx 100\%$ acceptance; kNN still recovers Mahalanobis-level ranking correlation, while the linear probe and CLIP alignment collapse on both axes. Mahalanobis held-out rows are reproduced from \Cref{tab:fpr95-selection}; full-support rows are kept here as sanity checks.}
\label{tab:fpr95-selection-ablation}
\small
\begin{tabular}{ll ccccc}
\rowcolor{gray!20}\toprule
Model & Setting & Accept\% & KID$_{\text{trust}}$$\downarrow$ & KID$_{\text{baseline}}$$\downarrow$ & $\Delta$\%$\uparrow$ & $\rho$(T)$\uparrow$ \\
\midrule
\multicolumn{7}{l}{\emph{CelebA} --- Mahalanobis (ours)} \\
Vanilla          & full   & 85.6 & 0.266{\tiny$\pm$.005} & 0.288{\tiny$\pm$.012} & $+$7.5  & $+$0.87 \\
Vanilla          & held-out & 58.7 & 0.224{\tiny$\pm$.009} & 0.368{\tiny$\pm$.020} & $+$39.1 & $+$0.96 \\
REPA (DINOv3)    & full   & 87.9 & 0.163{\tiny$\pm$.005} & 0.194{\tiny$\pm$.004} & $+$15.6 & $+$0.83 \\
REPA (DINOv3)    & held-out & 55.4 & 0.228{\tiny$\pm$.010} & 0.401{\tiny$\pm$.031} & $+$43.1 & $+$0.96 \\
REPA (SigLIP)    & full   & 85.2 & 0.225{\tiny$\pm$.009} & 0.265{\tiny$\pm$.008} & $+$15.1 & $+$0.88 \\
REPA (SigLIP)    & held-out & 56.1 & 0.239{\tiny$\pm$.018} & 0.423{\tiny$\pm$.020} & $+$43.6 & $+$0.97 \\
\addlinespace
\multicolumn{7}{l}{\emph{CelebA} --- Linear probe (per-attribute energy)} \\
Vanilla          & full   & 100.0 & 0.292{\tiny$\pm$.011} & 0.294{\tiny$\pm$.008} & $+$0.9 & $-$0.29 \\
Vanilla          & held-out & 100.0 & 0.350{\tiny$\pm$.012} & 0.343{\tiny$\pm$.033} & $-$2.2 & $-$0.04 \\
REPA (DINOv3)    & full   & 100.0 & 0.192{\tiny$\pm$.008} & 0.194{\tiny$\pm$.003} & $+$1.0 & $-$0.27 \\
REPA (DINOv3)    & held-out & 100.0 & 0.388{\tiny$\pm$.015} & 0.409{\tiny$\pm$.020} & $+$5.1 & $+$0.21 \\
REPA (SigLIP)    & full   & 100.0 & 0.256{\tiny$\pm$.013} & 0.256{\tiny$\pm$.013} & $+$0.0 & $-$0.29 \\
REPA (SigLIP)    & held-out & 100.0 & 0.391{\tiny$\pm$.016} & 0.417{\tiny$\pm$.033} & $+$6.2 & $+$0.34 \\
\addlinespace
\multicolumn{7}{l}{\emph{CelebA} --- kNN (per-attr., 5-th NN summed)} \\
Vanilla          & full   & 99.9 & 0.292{\tiny$\pm$.013} & 0.302{\tiny$\pm$.012} & $+$3.2 & $+$0.89 \\
Vanilla          & held-out & 99.8 & 0.357{\tiny$\pm$.023} & 0.358{\tiny$\pm$.021} & $+$0.3 & $+$0.92 \\
REPA (DINOv3)    & full   & 99.9 & 0.188{\tiny$\pm$.007} & 0.187{\tiny$\pm$.010} & $-$0.5 & $+$0.83 \\
REPA (DINOv3)    & held-out & 99.7 & 0.403{\tiny$\pm$.012} & 0.396{\tiny$\pm$.026} & $-$1.7 & $+$0.88 \\
REPA (SigLIP)    & full   & 99.9 & 0.263{\tiny$\pm$.011} & 0.263{\tiny$\pm$.010} & $+$0.2 & $+$0.91 \\
REPA (SigLIP)    & held-out & 99.8 & 0.424{\tiny$\pm$.020} & 0.411{\tiny$\pm$.024} & $-$3.0 & $+$0.87 \\
\addlinespace
\multicolumn{7}{l}{\emph{CelebA} --- CLIP alignment (16 joint-combo prompts)} \\
Vanilla          & full   & 100.0 & 0.649{\tiny$\pm$.013} & 0.656{\tiny$\pm$.008} & $+$1.0 & $-$0.55 \\
Vanilla          & held-out & 100.0 & 0.729{\tiny$\pm$.016} & 0.720{\tiny$\pm$.024} & $-$1.3 & $+$0.60 \\
REPA (DINOv3)    & full   & 100.0 & 0.625{\tiny$\pm$.012} & 0.603{\tiny$\pm$.015} & $-$3.6 & $-$0.48 \\
REPA (DINOv3)    & held-out & 100.0 & 0.770{\tiny$\pm$.020} & 0.775{\tiny$\pm$.034} & $+$0.7 & $+$0.77 \\
REPA (SigLIP)    & full   & 100.0 & 0.631{\tiny$\pm$.025} & 0.613{\tiny$\pm$.011} & $-$2.8 & $-$0.65 \\
REPA (SigLIP)    & held-out & 100.0 & 0.733{\tiny$\pm$.011} & 0.751{\tiny$\pm$.030} & $+$2.5 & $+$0.73 \\
\midrule
\multicolumn{7}{l}{\emph{RxRx1} (50-condition subset) --- Mahalanobis (ours)} \\
Vanilla          & full    & 14.9 & 0.202{\tiny$\pm$.003} & 0.307{\tiny$\pm$.018} & $+$34.2 & $+$0.87 \\
Vanilla          & held-out &  4.4 & 0.196{\tiny$\pm$.006} & 0.323{\tiny$\pm$.013} & $+$39.4 & $+$0.90 \\
REPA (DINOv3)    & full    & 21.7 & 0.175{\tiny$\pm$.004} & 0.276{\tiny$\pm$.013} & $+$36.7 & $+$0.83 \\
REPA (DINOv3)    & held-out &  5.5 & 0.197{\tiny$\pm$.006} & 0.333{\tiny$\pm$.019} & $+$40.8 & $+$0.90 \\
REPA (SigLIP)    & full    & 18.4 & 0.190{\tiny$\pm$.002} & 0.301{\tiny$\pm$.007} & $+$37.1 & $+$0.84 \\
REPA (SigLIP)    & held-out &  6.1 & 0.181{\tiny$\pm$.003} & 0.312{\tiny$\pm$.019} & $+$42.0 & $+$0.87 \\
\addlinespace
\multicolumn{7}{l}{\emph{RxRx1} --- Linear probe (per-attribute energy)} \\
Vanilla          & full    & 100.0 & 0.304{\tiny$\pm$.015} & 0.312{\tiny$\pm$.015} & $+$2.4 & $-$0.09 \\
Vanilla          & held-out & 100.0 & 0.329{\tiny$\pm$.020} & 0.330{\tiny$\pm$.023} & $+$0.5 & $+$0.08 \\
REPA (DINOv3)    & full    & 100.0 & 0.278{\tiny$\pm$.019} & 0.274{\tiny$\pm$.019} & $-$1.3 & $+$0.15 \\
REPA (DINOv3)    & held-out & 100.0 & 0.351{\tiny$\pm$.019} & 0.354{\tiny$\pm$.013} & $+$0.8 & $+$0.10 \\
REPA (SigLIP)    & full    & 100.0 & 0.303{\tiny$\pm$.011} & 0.311{\tiny$\pm$.011} & $+$2.7 & $+$0.01 \\
REPA (SigLIP)    & held-out & 100.0 & 0.330{\tiny$\pm$.025} & 0.318{\tiny$\pm$.018} & $-$3.9 & $+$0.08 \\
\addlinespace
\multicolumn{7}{l}{\emph{RxRx1} --- kNN (per-attr., 5-th NN summed)} \\
Vanilla          & full    & 100.0 & 0.304{\tiny$\pm$.015} & 0.312{\tiny$\pm$.015} & $+$2.4 & $+$0.91 \\
Vanilla          & held-out & 100.0 & 0.329{\tiny$\pm$.020} & 0.330{\tiny$\pm$.023} & $+$0.5 & $+$0.88 \\
REPA (DINOv3)    & full    & 100.0 & 0.278{\tiny$\pm$.019} & 0.274{\tiny$\pm$.019} & $-$1.3 & $+$0.92 \\
REPA (DINOv3)    & held-out & 100.0 & 0.351{\tiny$\pm$.019} & 0.354{\tiny$\pm$.013} & $+$0.8 & $+$0.91 \\
REPA (SigLIP)    & full    & 100.0 & 0.303{\tiny$\pm$.011} & 0.311{\tiny$\pm$.011} & $+$2.7 & $+$0.90 \\
REPA (SigLIP)    & held-out & 100.0 & 0.330{\tiny$\pm$.025} & 0.318{\tiny$\pm$.018} & $-$3.9 & $+$0.89 \\
\bottomrule
\end{tabular}
\end{table}

\section{Detailed results}
\label{app:detailed_results}

This appendix collects the auxiliary results referenced from \Cref{sec:fpr95,sec:downstream-ordering,sec:during-gen}: full-support sanity-check rows for the P95-real-threshold selection, condition-level Spearman correlations of the trust score with its realism and faithfulness components, the trust/realism/faithfulness decile decomposition on \celeba, the analogous DINOv3 decile binning on \rxrx, and the during-generation decile binning under the \transl.

\paragraph{Full-support sanity checks.}
\Cref{tab:fpr95-selection-full} extends \Cref{tab:fpr95-selection} with the full-support training-distribution rows. Relative KID gains shrink under full support, where the generator already covers the target conditions, but remain positive on both datasets and all three generators ($+$7--16\% on \celeba, $+$34--37\% on \rxrx). The held-out rows are the deployment-relevant ones and are reproduced compactly in \Cref{tab:fpr95-selection}.

\begin{table}[h]
\centering
\caption{Complete P95-real-threshold post-generation sample selection (DINOv3 / SigLIP trust scoring). Full-support rows are sanity checks; held-out rows are the main deployment-relevant setting and are reproduced compactly in \Cref{tab:fpr95-selection}. Threshold set at the 95th percentile of real-sample trust scores; for marginal models, only seen-combo real samples are used for calibration. KID measured in DINOv3 space via shuffled subsampling ($k{=}500$, repeated for CIs). Random baseline: random subset of the same size, each compared against its condition-matched real pool. $\Delta$\% = relative KID improvement of trust-selected over random (positive = better).}
\label{tab:fpr95-selection-full}
\small
\begin{tabular}{ll cccc}
\rowcolor{gray!20}\toprule
Model & Setting & Accept\% & KID$_{\text{trust}}$$\downarrow$ & KID$_{\text{baseline}}$$\downarrow$ & $\Delta$\%$\uparrow$ \\
\midrule
\multicolumn{6}{l}{\emph{CelebA}} \\
Vanilla          & full   & 85.6 & 0.266{\tiny$\pm$.005} & 0.288{\tiny$\pm$.012} & $+$7.5  \\
Vanilla          & held-out & 58.7 & 0.224{\tiny$\pm$.009} & 0.368{\tiny$\pm$.020} & $+$39.1 \\
REPA (DINOv3)    & full   & 87.9 & 0.163{\tiny$\pm$.005} & 0.194{\tiny$\pm$.004} & $+$15.6 \\
REPA (DINOv3)    & held-out & 55.4 & 0.228{\tiny$\pm$.010} & 0.401{\tiny$\pm$.031} & $+$43.1 \\
REPA (SigLIP)    & full   & 85.2 & 0.225{\tiny$\pm$.009} & 0.265{\tiny$\pm$.008} & $+$15.1 \\
REPA (SigLIP)    & held-out & 56.1 & 0.239{\tiny$\pm$.018} & 0.423{\tiny$\pm$.020} & $+$43.6 \\
\midrule
\multicolumn{6}{l}{\emph{RxRx1} (50-condition subset)} \\
Vanilla          & full    & 14.9 & 0.202{\tiny$\pm$.003} & 0.307{\tiny$\pm$.018} & $+$34.2 \\
Vanilla          & held-out &  4.4 & 0.196{\tiny$\pm$.006} & 0.323{\tiny$\pm$.013} & $+$39.4 \\
REPA (DINOv3)    & full    & 21.7 & 0.175{\tiny$\pm$.004} & 0.276{\tiny$\pm$.013} & $+$36.7 \\
REPA (DINOv3)    & held-out &  5.5 & 0.197{\tiny$\pm$.006} & 0.333{\tiny$\pm$.019} & $+$40.8 \\
REPA (SigLIP)    & full    & 18.4 & 0.190{\tiny$\pm$.002} & 0.301{\tiny$\pm$.007} & $+$37.1 \\
REPA (SigLIP)    & held-out &  6.1 & 0.181{\tiny$\pm$.003} & 0.312{\tiny$\pm$.019} & $+$42.0 \\
\bottomrule
\end{tabular}
\end{table}

\paragraph{Component-level condition correlations.}
\Cref{tab:condition-correlation} reports the per-condition Spearman correlations of trust, realism, and faithfulness against the held-out $\Delta$KID. Trust is the strongest predictor across both datasets and all three generators ($\rho(\trust)\!\geq\!0.88$). The decomposition shows that on \rxrx the realism signal carries most of the predictive power, consistent with the held-out perturbations producing manifold drift that the generator cannot reach; on \celeba both faithfulness and realism contribute, with faithfulness being the slightly stronger of the two on the controlled compositional split.

\begin{table}[h]
\centering
\caption{Condition-level Spearman correlations under the harder support-shift regime: CelebA held-out and RxRx1 held-out (DINOv3/SigLIP scoring, marginal models). $\rho(\trust)$ = trust, $\rho$(R) = realism, $\rho$(F) = faithfulness. The main text reports only $\rho(\trust)$ in \Cref{tab:fpr95-selection}.}
\label{tab:condition-correlation}
\small
\begin{tabular}{l ccc ccc}
\rowcolor{gray!20}\toprule
 & \multicolumn{3}{c}{CelebA ($n{=}16$)} & \multicolumn{3}{c}{RxRx1 ($n{=}50$)} \\
\cmidrule(lr){2-4} \cmidrule(lr){5-7}
Model & $\rho(\trust)$$\uparrow$ & $\rho$(R)$\uparrow$ & $\rho$(F)$\uparrow$ & $\rho(\trust)$$\uparrow$ & $\rho$(R)$\uparrow$ & $\rho$(F)$\uparrow$ \\
\midrule
Vanilla       & 0.96 & 0.69 & 0.84 & 0.90 & 0.88 & 0.68 \\
REPA (DINOv3) & 0.96 & 0.71 & 0.86 & 0.90 & 0.88 & 0.75 \\
REPA (SigLIP) & 0.96 & 0.76 & 0.84 & 0.88 & 0.86 & 0.75 \\
\bottomrule
\end{tabular}
\end{table}

\paragraph{Trust/realism/faithfulness decile decomposition (\celeba).}
\Cref{fig:celeba-decile-decomposition} re-plots the main-text \Cref{fig:celeba-decile} with all three score components. Realism alone does not track downstream condition accuracy, consistent with realism being a manifold-quality rather than a condition-aware signal; trust combines the two and drives the ordering on both panels.

\begin{figure}[h]
\centering
\begin{subfigure}[t]{0.48\textwidth}
    \centering
    \includegraphics[width=\textwidth]{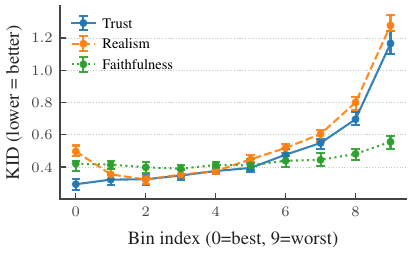}
    \caption{CelebA $\Delta$KID by decile.}
    \label{fig:celeba-decile-kid-decomp}
\end{subfigure}
\hfill
\begin{subfigure}[t]{0.48\textwidth}
    \centering
    \includegraphics[width=\textwidth]{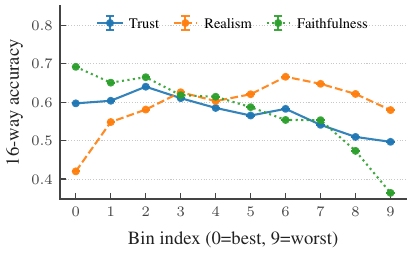}
    \caption{CelebA downstream 16-way accuracy.}
    \label{fig:celeba-decile-downstream-decomp}
\end{subfigure}
\caption{\textbf{Trust / realism / faithfulness decomposition for the main-text \Cref{fig:celeba-decile}.} Same setting as the main figure (REPA-DINOv3 held-out, DINOv3 scoring), but plotting all three score components. Realism does not track downstream condition accuracy, which is consistent with realism not being a condition-aware signal; trust drives the ordering on both panels.}
\label{fig:celeba-decile-decomposition}
\end{figure}

\paragraph{\rxrx DINOv3 decile binning.}
\Cref{fig:rxrx1-decile} shows the analogous DINOv3-space decile binning on the \rxrx 50-condition subset. Better trust deciles produce both lower $\Delta$KID and stronger downstream classifiers, mirroring the \celeba ordering. The main-text discussion emphasizes the CellProfiler morphology validation in \Cref{sec:cp-validation} because CP features are independent of the learned DINOv3 validation encoder used for KID.

\begin{figure}[h]
\centering
\begin{subfigure}[t]{0.48\textwidth}
    \centering
    \includegraphics[width=\textwidth]{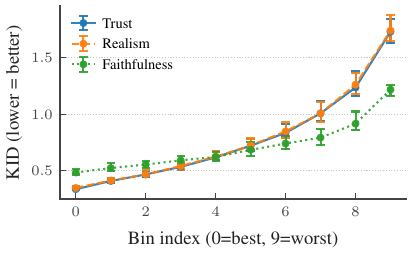}
    \caption{RxRx1 $\Delta$KID by trust decile.}
    \label{fig:rxrx1-decile-kid}
\end{subfigure}
\hfill
\begin{subfigure}[t]{0.48\textwidth}
    \centering
    \includegraphics[width=\textwidth]{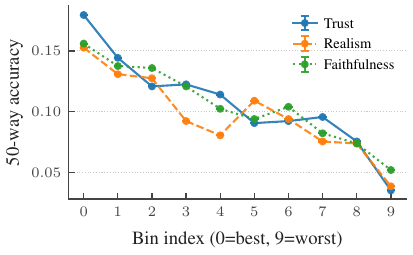}
    \caption{RxRx1 downstream classification.}
    \label{fig:rxrx1-decile-downstream}
\end{subfigure}
\caption{\textbf{RxRx1 DINOv3 decile binning (REPA-DINOv3 held-out, DINOv3 scoring, 50-condition subset).} These learned-encoder trends support the same ordering story as CelebA, but the main text emphasizes the CellProfiler morphology validation because it is independent of the learned DINOv3 validation encoder.}
\label{fig:rxrx1-decile}
\end{figure}

\paragraph{During-generation decile binning.}
\Cref{fig:aligned-decile} confirms that the \transl features reproduce the same monotonic trust ordering on both \celeba and \rxrx without decoding the sample, complementing the support-shift evidence in \Cref{tab:aligned-combined}.

\begin{figure}[h]
\centering
\begin{subfigure}[t]{0.48\textwidth}
    \centering
    \includegraphics[width=\textwidth]{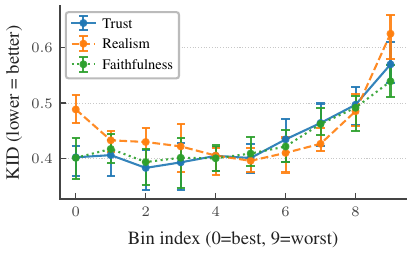}
    \caption{CelebA $\Delta$KID by decile.}
    \label{fig:celeba-aligned-kid}
\end{subfigure}
\hfill
\begin{subfigure}[t]{0.48\textwidth}
    \centering
    \includegraphics[width=\textwidth]{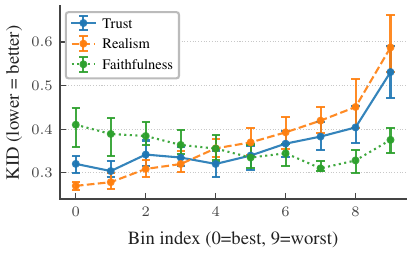}
    \caption{RxRx1 $\Delta$KID by decile.}
    \label{fig:rxrx1-aligned-kid}
\end{subfigure}
\caption{\textbf{During-generation decile binning after translation.} The \transl features recover the monotonic trust trend without decoding the sample and re-encoding it through the \fext.}
\label{fig:aligned-decile}
\end{figure}

\section{Scorer-design ablation}
\label{app:scorer-ablation}

\subsection{Complete during-generation table}
\label{app:during-gen-full}

\Cref{tab:aligned-combined-full} reports the full during-generation table, including full-support sanity-check rows that are omitted from the main text. The main text keeps only the held-out rows because those are aligned with the missing-target evaluation setting.

\begin{table}[H]
\centering
\caption{Complete non-DINOv3 trust scoring: model-internal REPA features and our \transl. \emph{Left}: P95-real-threshold selection (KID in DINOv3 space, random baseline with per-draw condition-matched real). \emph{Right}: condition-level Spearman correlations. Full-support rows are sanity checks; held-out rows are summarized in \Cref{tab:aligned-combined}.}
\label{tab:aligned-combined-full}
\small
\begin{tabular}{ll cccc ccc}
\rowcolor{gray!20}\toprule
 & & \multicolumn{4}{c}{P95-real Selection} & \multicolumn{3}{c}{Correlations} \\
\cmidrule(lr){3-6} \cmidrule(lr){7-9}
Model & Setting & Accept\% & KID$_\text{trust}$$\downarrow$ & KID$_\text{baseline}$$\downarrow$ & $\Delta$\%$\uparrow$ & $\rho(\trust)$$\uparrow$ & $\rho$(R)$\uparrow$ & $\rho$(F)$\uparrow$ \\
\midrule
\multicolumn{9}{l}{\emph{CelebA --- REPA aligned\_mean (during generation)}} \\
REPA (DINOv3) & full    & 97.8 & 0.187{\tiny$\pm$.013} & 0.195{\tiny$\pm$.010} & $+$4.1  & 0.73 & 0.74 & 0.29 \\
REPA (DINOv3) & held-out  & 62.7 & 0.278{\tiny$\pm$.015} & 0.404{\tiny$\pm$.021} & $+$31.2 & 0.95 & 0.64 & 0.91 \\
REPA (SigLIP) & full    & 97.2 & 0.254{\tiny$\pm$.008} & 0.258{\tiny$\pm$.012} & $+$1.6  & 0.81 & 0.80 & 0.53 \\
REPA (SigLIP) & held-out  & 66.7 & 0.293{\tiny$\pm$.024} & 0.425{\tiny$\pm$.018} & $+$31.1 & 0.95 & 0.56 & 0.89 \\
\midrule
\multicolumn{9}{l}{\emph{CelebA, \transl (during generation)}} \\
Vanilla       & full    & 32.6 & 0.254{\tiny$\pm$.015} & 0.299{\tiny$\pm$.007} & $+$15.2 & 0.61 & 0.29 & 0.57 \\
Vanilla       & held-out  & 23.6 & 0.195{\tiny$\pm$.008} & 0.378{\tiny$\pm$.023} & $+$48.4 & 0.83 & 0.74 & 0.82 \\
REPA (DINOv3) & full    & 89.2 & 0.188{\tiny$\pm$.010} & 0.197{\tiny$\pm$.005} & $+$4.4  & 0.78 & 0.64 & 0.71 \\
REPA (DINOv3) & held-out  & 48.4 & 0.250{\tiny$\pm$.016} & 0.397{\tiny$\pm$.017} & $+$37.2 & 0.92 & 0.66 & 0.91 \\
REPA (SigLIP) & full    & 87.2 & 0.248{\tiny$\pm$.009} & 0.263{\tiny$\pm$.011} & $+$5.7  & 0.56 & 0.29 & 0.52 \\
REPA (SigLIP) & held-out  & 17.2 & 0.308{\tiny$\pm$.015} & 0.409{\tiny$\pm$.015} & $+$24.8 & 0.88 & 0.83 & 0.86 \\
\midrule
\multicolumn{9}{l}{\emph{RxRx1 --- REPA aligned\_mean (during generation)}} \\
REPA (DINOv3) & full    & 26.4 & 0.378{\tiny$\pm$.018} & 0.283{\tiny$\pm$.005} & $-$33.4 & $-$0.01 & 0.10  & $-$0.27 \\
REPA (DINOv3) & held-out &  1.4 & 0.537{\tiny$\pm$.032} & 0.329{\tiny$\pm$.037} & $-$63.4 & 0.09    & 0.05  & 0.19 \\
REPA (SigLIP) & full    & 20.6 & 0.628{\tiny$\pm$.029} & 0.315{\tiny$\pm$.006} & $-$99.5 & $-$0.13 & $-$0.08 & $-$0.38 \\
REPA (SigLIP) & held-out &  3.9 & 0.564{\tiny$\pm$.022} & 0.293{\tiny$\pm$.023} & $-$92.3 & 0.13    & $-$0.11 & 0.55 \\
\midrule
\multicolumn{9}{l}{\emph{RxRx1, \transl (during generation)}} \\
Vanilla       & full    & 37.0 & 0.277{\tiny$\pm$.007} & 0.290{\tiny$\pm$.017} & $+$4.7  & 0.70 & 0.36 & 0.67 \\
Vanilla       & held-out & 38.4 & 0.260{\tiny$\pm$.003} & 0.346{\tiny$\pm$.016} & $+$24.9 & 0.78 & 0.31 & 0.76 \\
REPA (DINOv3) & full    & 82.7 & 0.276{\tiny$\pm$.017} & 0.281{\tiny$\pm$.016} & $+$1.9  & 0.18    & 0.67 & $-$0.17 \\
REPA (DINOv3) & held-out & 62.3 & 0.246{\tiny$\pm$.003} & 0.372{\tiny$\pm$.037} & $+$33.8 & 0.88 & 0.47 & 0.80 \\
REPA (SigLIP) & full    & 80.3 & 0.290{\tiny$\pm$.015} & 0.316{\tiny$\pm$.016} & $+$8.1  & 0.65 & 0.68 & 0.00 \\
REPA (SigLIP) & held-out & 60.5 & 0.235{\tiny$\pm$.011} & 0.329{\tiny$\pm$.018} & $+$28.7 & 0.70 & 0.31 & 0.68 \\
\bottomrule
\end{tabular}
\end{table}

\subsection{Fine-grained timestep (full table)}
\label{app:timestep-fine}

\Cref{tab:celeba-timestep-fine} gives the full per-$k$ numbers summarized by \Cref{fig:celeba-timestep-fine} in the main text.

\begin{table}[H]
    \centering
    \small
    \setlength{\tabcolsep}{5pt}
\begin{tabular}{lccccccc}
\rowcolor{gray!20}\toprule
Scoring step & $\rho_{\text{trust}}$ & KID$_{\text{trust}}$$\downarrow$ & KID$_{\text{baseline}}$$\downarrow$ & $\Delta$KID\%$\uparrow$ (accept) & Image $\Delta$ L2 & Steps saved \\
\midrule
$k{=}0$\,($t{\approx}1.00$)  & $0.21$ & 0.290 & 0.358 & +18.8\% (19\%) & -- & ${\approx}100\%$ \\
$k{=}27$\,($t{\approx}0.90$)  & $0.08$ & 0.309 & 0.365 & +15.5\% (19\%) & 74.5 & ${\approx}89\%$ \\
$k{=}55$\,($t{\approx}0.79$)  & $0.37$ & 0.263 & 0.356 & +26.1\% (19\%) & 33.7 & ${\approx}78\%$ \\
$k{=}83$\,($t{\approx}0.68$)  & $0.58$ & 0.240 & 0.357 & +33.0\% (19\%) & 24.2 & ${\approx}67\%$ \\
$k{=}110$\,($t{\approx}0.58$)  & $0.65$ & 0.234 & 0.355 & +34.1\% (21\%) & 19.3 & ${\approx}56\%$ \\
$k{=}138$\,($t{\approx}0.47$)  & $0.73$ & 0.223 & 0.346 & +35.5\% (21\%) & 17.9 & ${\approx}45\%$ \\
$k{=}166$\,($t{\approx}0.36$)  & $0.75$ & 0.208 & 0.342 & +39.1\% (21\%) & 16.7 & ${\approx}34\%$ \\
$k{=}193$\,($t{\approx}0.26$)  & $0.79$ & 0.207 & 0.367 & +43.6\% (22\%) & 15.6 & ${\approx}23\%$ \\
$k{=}221$\,($t{\approx}0.15$)  & $0.78$ & 0.202 & 0.356 & +43.1\% (23\%) & 14.7 & ${\approx}12\%$ \\
$k{=}248$\,($t{\approx}0.04$)  & $0.79$ & 0.210 & 0.347 & +39.4\% (24\%) & 12.3 & ${\approx}1\%$ \\
\midrule
Post generation  & $0.96$ & 0.221 & 0.365 & +39.3\% (59\%) & -- & --- \\
\bottomrule
\end{tabular}
    \caption{\textbf{Fine-grained timestep ablation for cheap internal trust scoring on CelebA held-out} (Vanilla SiT-B/2, 250-step sampler). At each scoring step $k$, the predicted clean latent $\hat{x}_0$ is passed through the \transl into SigLIP space (no VAE decode, no encoder pass) before Mahalanobis scoring. Spearman $\rho_{\text{trust}}$ is the condition-level correlation against $\Delta$KID. KID$_{\text{acc}}$ and KID$_{\text{baseline}}$ are computed on the P95-real-threshold-accepted set and random matched-size subsets, respectively. ``Image $\Delta$ L2'' is the mean L2 distance between consecutively captured VAE-decoded $\hat{x}_0$ images, quantifying how much the predicted clean image still moves between capture steps (undefined at $k{=}0$). ``Steps saved'' is the fraction of sampler steps avoided by abstaining at $k$. The final row is the post-generation DINOv3 oracle (reproduced from \Cref{tab:fpr95-selection-ablation}, Vanilla held-out). The ranking signal grows monotonically with $k$, and P95-real-threshold $\Delta$KID matches or exceeds the post-generation oracle from $k{\approx}166$ onward while still saving ${\geq}34\%$ of generation compute.}
    \label{tab:celeba-timestep-fine}
\end{table}

\subsection{RxRx1 timestep ablation}
\label{app:timestep-rxrx1}

\begin{wrapfigure}[25]{r}{0.5\textwidth}
  \includegraphics[width=\linewidth]{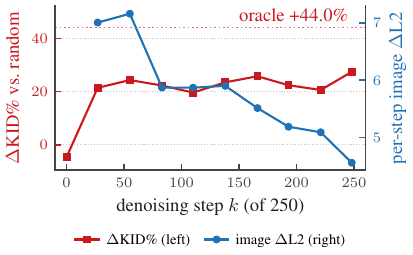}
  \caption{\textbf{\rxrx timestep ablation (Vanilla SiT B/2, 250-step sampler, \transl features).} Red squares: P95-real $\Delta$KID\% of the FPR95-accepted subset against a condition-matched random subset, evaluated on intermediate predicted-clean latents $\hat x_0(k)$ projected through the \transl. Blue circles: per-step L2 change in the VAE-decoded $\hat x_0$. Dashed red line: post-generation DINOv3 oracle ($+44.0\%$). Compared to \Cref{fig:celeba-timestep-fine}, the during generation signal saturates earlier ($k\!\approx\!27$) but plateaus below the post-generation oracle, reflecting the harder support shift on \rxrx.}
  \label{fig:rxrx1-timestep-fine}
\end{wrapfigure}

The same fine-grained timestep ablation on \rxrx held-out is shown in \Cref{fig:rxrx1-timestep-fine} for the Vanilla SiT-B/2 generator with the \transl scoring path and the same 250-step sampler used for \celeba. The qualitative picture matches \Cref{fig:celeba-timestep-fine}: the trust signal is already useful well before the final decoded sample, and the per-step L2 change in the predicted-clean image $\hat x_0$ falls monotonically as the trajectory settles ($\sim 7.0$ at $k\!=\!27$ down to $\sim 4.6$ at $k\!=\!248$).

Two differences are worth flagging. First, on \rxrx the during-generation $\Delta$KID saturates very early --- by $k\!=\!27$ the FPR95-selected subset already reaches roughly $+21\%$ over a condition-matched random subset, and subsequent steps oscillate in the $+20\%$ to $+27\%$ band without a further monotonic trend. Second, the during-generation score does not match the post-generation DINOv3 oracle ($+44.0\%$, dashed line): on this harder support-shift regime the mapped internal feature captures roughly $60\%$ of the discriminative power that the full \fext pass recovers after decoding. This gap is consistent with the regime in which Mahalanobis-based scoring is intrinsically harder on \rxrx (smaller per-attribute margins on the sirna axis, \Cref{tab:a1-diagnostic,tab:pooled-reference-certification}); the during-generation extension preserves a useful trust ordering but does not recover the full precision of the post-generation oracle. The practical compute trade-off is therefore sharper than on \celeba: abstaining as early as $k\!=\!27$ saves ${\approx}89\%$ of the denoising trajectory while retaining most of the during-generation $\Delta$KID signal, but full post-generation scoring remains the higher-precision choice when compute is not the bottleneck.



\end{document}